\documentclass[letterpaper]{article} 
\usepackage{aaai2026}  
\usepackage{times}  
\usepackage{helvet}  
\usepackage{courier}  
\usepackage[hyphens]{url}  
\usepackage{graphicx} 
\usepackage{natbib}  
\usepackage{caption} 
\usepackage{algorithm}
\usepackage{algorithmic}

\usepackage{newfloat}
\usepackage{listings}
\DeclareCaptionStyle{ruled}{labelfont=normalfont,labelsep=colon,strut=off} 
\floatstyle{ruled}
\newfloat{listing}{tb}{lst}{}
\floatname{listing}{Listing}
\title{Attention Basin: Why Contextual Position Matters in Large Language Models}
\author{
    Zihao Yi\textsuperscript{\rm 1},
    Delong Zeng\textsuperscript{\rm 1, \equalcontrib},
    Zhenqing Ling\textsuperscript{\rm 1, \equalcontrib},
    Haohao Luo\textsuperscript{\rm 1},
    Zhe Xu\textsuperscript{\rm 1},
    Wei Liu\textsuperscript{\rm 2},
    Jian Luan\textsuperscript{\rm 2},
    Wanxia Cao\textsuperscript{\rm 2},
    and Ying Shen\textsuperscript{\rm 1, \dag}
}

\affiliations{
    \textsuperscript{\rm 1}Sun Yat-sen University, Shenzhen, China \\
    \textsuperscript{\rm 2}MiLM Plus, Xiaomi Inc., Beijing, China
     
}

\usepackage{bibentry}

\usepackage[utf8]{inputenc} 
\usepackage[T1]{fontenc}    
\usepackage{url}            
\usepackage{booktabs}       
\usepackage{amsfonts}       
\usepackage{nicefrac}       
\usepackage{microtype}      
\usepackage{xcolor}         
\usepackage{multirow}
\usepackage{multicol}
\usepackage{graphicx}
\usepackage{amsmath}
\usepackage{makecell}
\usepackage{graphicx}
\usepackage{subcaption}
\usepackage{enumitem}
\usepackage{xspace}
\usepackage{amsthm}

\begin{document}

\maketitle
\newcommand{\zhenqing}[1]{[\textcolor{blue}{#1}]}
\newcommand{\ours}{\textbf{AttnRank}\xspace}
\theoremstyle{plain}
\newtheorem{theorem}{Theorem}[section]
\newtheorem{proposition}[theorem]{Proposition}
\newtheorem{lemma}[theorem]{Lemma}
\newtheorem{corollary}[theorem]{Corollary}
\theoremstyle{definition}
\newtheorem{definition}[theorem]{Definition}
\newtheorem{assumption}[theorem]{Assumption}
\theoremstyle{remark}
\newtheorem{remark}[theorem]{Remark}

\begin{abstract}

The performance of Large Language Models (LLMs) is significantly sensitive to the contextual position of information in the input. To investigate the mechanism behind this positional bias, our extensive experiments reveal a consistent phenomenon we term the \emph{attention basin}: when presented with a sequence of structured items (e.g., retrieved documents or few-shot examples), models systematically assign higher attention to the items at the beginning and end of the sequence, while neglecting those in the middle. Crucially, our analysis further reveals that allocating higher attention to critical information is key to enhancing model performance. Based on these insights, we introduce Attention-Driven Reranking (\ours), a two-stage framework that (i) estimates a model's intrinsic positional attention preferences using a small calibration set, and (ii) reorders retrieved documents or few-shot examples to align the most salient content with these high-attention positions. \ours{} is a model-agnostic, training-free, and plug-and-play method with minimal computational overhead. Experiments on multi-hop QA and few-shot in-context learning tasks demonstrate that \ours{} achieves substantial improvements across 10 large language models of varying architectures and scales, without modifying model parameters or training procedures.


\end{abstract}
\section{Introduction}
\label{sec:intro}
Large Language Models (LLMs) have acquired vast amounts of knowledge from large-scale pretraining corpora and demonstrated exceptional capabilities in understanding and generating natural language. As a result, they have achieved remarkable success across a wide range of language tasks, from text summarization to multi-turn dialogue \cite{yi2024survey,zhang2025rule,ling2025enhancing}. With the growing ability of LLMs to process increasingly long input sequences, Retrieval-Augmented Generation (RAG) has emerged as an effective paradigm for expanding the knowledge boundaries of LLMs and improving answer accuracy. By providing relevant external documents at inference time, RAG enables models to perform competitively even in complex scenarios such as multi-hop question answering and few-shot learning \cite{li2024meqa,yi2025intent}.

However, behind this apparent success lies a fundamental vulnerability: the performance of LLMs is highly sensitive to the position of information within the input context \cite{liu2024lost}. This positional bias poses a critical bottleneck—models often fail to \emph{utilize} long contexts effectively. Even when all necessary information is present, performance can degrade significantly if key content is placed in regions of low attention, leading to suboptimal and unpredictable outcomes \cite{xiao2023efficient}.


This vulnerability manifests empirically as the "lost-in-the-middle" (LIM) phenomenon, where models show a clear preference for information at the beginning or end of the context \cite{liu2024lost}. Although LIM provides an insightful phenomenological analysis, it describes the effect rather than elucidating the underlying cause. Concurrently, other mitigation strategies, such as fine-tuning for better instruction following \cite{li2023long}, explore different avenues but often incur substantial computational overhead without addressing the core mechanism of the bias.

To move beyond surface-level symptoms and uncover the root cause of positional bias, we turn our investigation to the very core of how LLMs process context: the attention mechanism \cite{vaswani2017attention}. To this end, our empirical analysis in a meticulously designed experimental sandbox across 10 mainstream LLMs reveals a consistent and systematic pattern, which we term the \textit{attention basin}: attention is disproportionately concentrated at the boundaries of the overall structural block of context, such as the full set of retrieved documents, creating a trough in the middle where information is neglected. Guided by this discovery, we conduct a theoretical investigation to establish a direct link between this skewed attention allocation and the model's final output probabilities, confirming that placing critical information in high-attention zones is paramount for effective context utilization.

Armed with this mechanistic insight, we propose Attention-Driven Reranking (\ours), a lightweight, training-free yet powerful two-stage method to mitigate positional bias at inference time. First, we probe the model’s intrinsic attention landscape using a small, representative calibration set. Second, we leverage this map to re-rank the input context, strategically aligning the most critical information with the model's natural high-attention regions. Extensive evaluations on multi-hop QA and few-shot learning tasks show that \ours consistently improves performance across 10 mainstream LLMs across varying architectures and scales, achieving significant gains without any model training or parameter modification.



This work makes the following key contributions:
\begin{itemize}[leftmargin=*]

\item  \textbf{Uncovering the Mechanism of Positional Bias:} We empirically and theoretically identify the \emph{attention basin} phenomenon as a core mechanistic driver of positional bias: LLMs intrinsically allocate higher attention to the start and end of the overall structural block of context, aligning critical information with high-attention zones is crucial for effective context utilization.


\item \textbf{Introducing Attention-Driven Reranking Method:} We propose a novel, lightweight and training-free method \ours that first maps a model's inherent positional attention preferences and then reorders the input to exploit these preferences for improved performance.

\item \textbf{Validation of \ours's Effectiveness:} Extensive experiments on multi-hop QA and few-shot learning tasks demonstrate that \ours consistently outperforms baseline strategies across 10 mainstream LLMs of varying architectures and scales, effectively mitigating positional bias and significantly enhancing information utilization.
\end{itemize}

\section{Related Works}
\label{sec:related_work}

The sensitivity of LLMs to the positional placement of information within their context is a well-documented vulnerability. This positional bias compromises their robustness, critically undermining reliability in applications like Retrieval-Augmented Generation (RAG) \cite{lewis2020retrieval} and in-context learning \cite{brown2020language}. This section reviews the two primary research thrusts addressing this issue: characterizing the phenomenon and mitigating its effects.

\paragraph{Characterizing Positional Bias.}
Initial research focused on characterizing the observable effects of positional bias. The most prominent finding is the "Lost in the Middle" (LIM) effect, where models exhibit a U-shaped performance curve, recalling information from the beginning and end of a context far more effectively than from the middle \cite{liu2024lost}. While a powerful phenomenological description, the LIM model does not elucidate the mechanistic underpinnings of this behavior. Subsequent work has sought these mechanisms. The "attention sink" phenomenon, for instance, reveals that models disproportionately attend to initial tokens \cite{xiao2023efficient}. However, this only explains the primacy effect and is limited to a token-level analysis, often on semantically sparse tokens (e.g., start-of-sequence markers), failing to provide a comprehensive model for positional effects across the entire input. An alternative hypothesis suggests the bias stems from a lack of explicit supervision during training, which fails to teach models that all positions are equally relevant \cite{an2024make}. This high-level explanation, however, lacks a concrete, testable mechanism and does not offer a methodology to quantify and compare the degree of bias across different models, leaving a significant gap in both fundamental understanding and empirical evaluation.

\paragraph{Mitigating Positional Bias.}
In response to identified biases, context reranking has emerged as a primary mitigation strategy. The simplest approaches employ brittle and ad-hoc heuristics, such as manually moving important documents to the context's edges \cite{liu2024lost}. While straightforward, their effectiveness is inconsistent across different models and tasks. More sophisticated methods use an auxiliary LLM as a reranker to optimize the context order before processing \cite{wang2025vidorag, 2025arXiv250722050G, zhang2025qwen3}. These LLM-based rerankers can improve performance but introduce significant computational overhead and latency, rendering them impractical for many real-time applications. More recently, training-based solutions have demonstrated the feasibility of directly reducing bias. For example, the information-intensive (IN2) training scheme by \citet{an2024make} effectively alleviates positional bias. However, such methods require a full fine-tuning process, which is computationally expensive and risks degrading the model's general capabilities on other tasks.

In summary, while prior work has successfully identified positional bias and proposed preliminary solutions, a critical gap persists. Our understanding of the underlying mechanisms remains incomplete, and existing mitigation strategies force an undesirable trade-off between unreliable heuristics and computationally expensive methods. This paper bridges this gap by first conducting a deeper mechanistic analysis of positional bias. Based on these insights, we introduce a novel, lightweight strategy that is both principled and efficient, offering a more robust and scalable solution than has been previously available.
\section{Observation and Analysis}
\label{sec:observation}

\begin{figure*}[!h]
  \centering
  \begin{subfigure}[b]{0.48\linewidth}
    \centering
    \includegraphics[width=\linewidth]{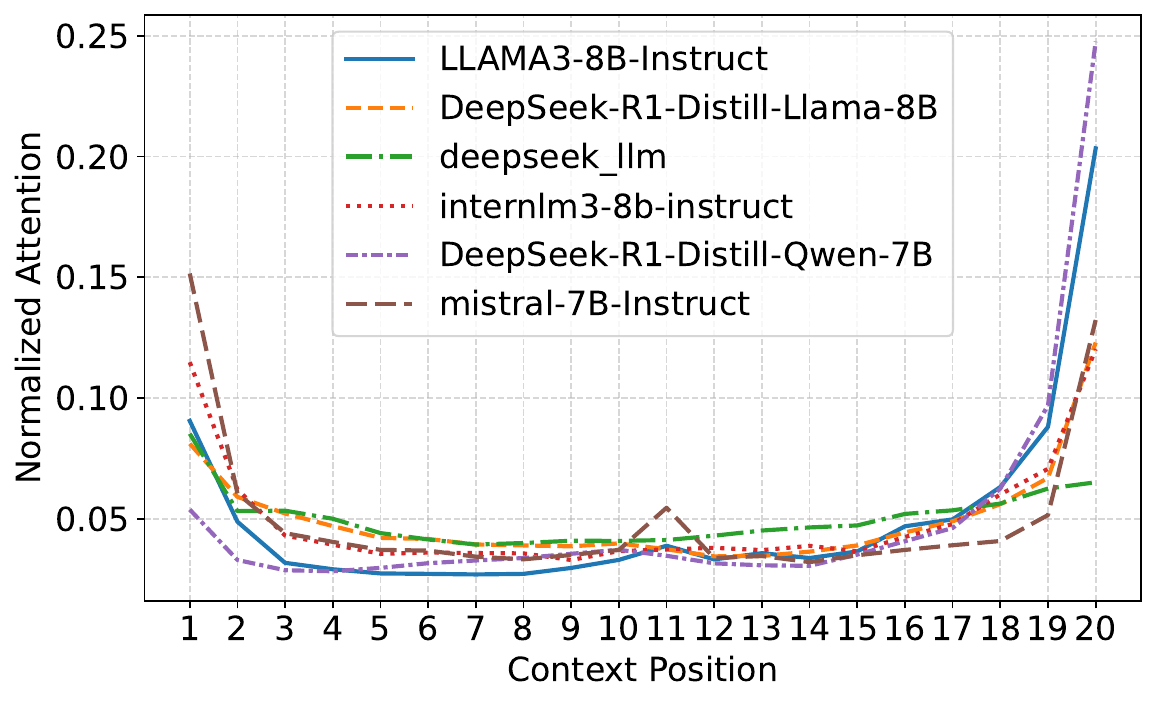}
    \caption{Attention Basin Phenomenon}
    \label{fig:attention-distribution0}
  \end{subfigure}
  \hfill
  \begin{subfigure}[b]{0.48\linewidth}
    \centering
    \includegraphics[width=\linewidth]{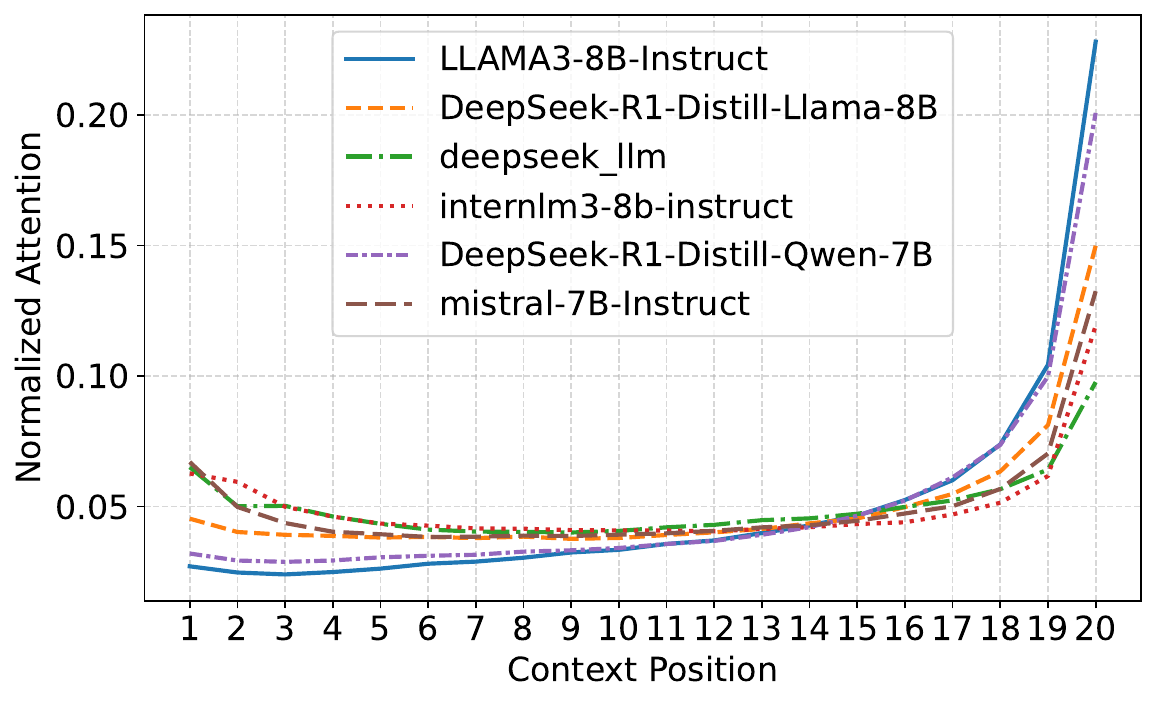}
    \caption{Effect of Disrupted Delimiters}
    \label{fig:attention-distribution1}
  \end{subfigure}
  \caption{(a) The model exhibits a U-shaped attention distribution, prioritizing tokens at the context's start and end boundaries. (b) This pattern disappears following the removal of structural delimiters, indicating that the phenomenon stems from the model's awareness of coherent segment boundaries.}
  \label{fig:attention-distribution}
\end{figure*}


In this section, we present our key discoveries that illuminate the mechanisms behind the performance variations of Large Language Models (LLMs) due to contextual position. Our analysis unfolds in three parts: first, we identify a consistent, structure-aware attention pattern we term the \textit{attention basin}; second, we establish a direct link between this attention pattern and model performance; and third, we pinpoint which layers are most critical in forming this positional preference.

\subsection{The \emph{Attention Basin} Phenomenon}
\label{attention_basin}

Inputs to LLMs for tasks like RAG or few-shot learning often consist of multiple, distinct segments. We model such inputs as a sequence of structural blocks 
$S=\{t, d_1, \ldots, d_k, q\}$, where $t$ is a task-defining template, $d_i$ are semantically coherent blocks of content (e.g., retrieved documents), and 
$q$ is the user's query.

While prior work has noted that attention is often high on initial tokens and immediate neighbors \cite{ma2024compressing}, the distribution of attention at a macro level across these structural blocks has been largely overlooked. To investigate this, we designed a precise experiment. For each input, we first located the token indices in the attention matrix corresponding to each document block $D = \{d_1, \ldots, d_k\}$. We then calculated the mean attention score allocated from the query tokens to each document's specific index range.

As shown in Figure~\ref{fig:attention-distribution0}, this quantitative analysis unearthed a consistent and striking pattern across all tested LLMs: the model allocates significantly higher attention to documents located towards the beginning and the end of the input sequence. We term this U-shaped distribution the \emph{attention basin} phenomenon.

This discovery raises a critical question: Is the attention basin a simple artifact of absolute position (i.e., the model just likes the start and end of any text), or is it driven by the model's perception of the input's structure? To distinguish between these possibilities, we designed a disruption experiment aimed at understanding the root cause of this phenomenon. We systematically dismantled the input's structure by removing punctuation, capitalization, and explicit delimiters like "Document [1]", effectively blending the distinct documents into a single, unstructured block of text.

As shown in Figure~\ref{fig:attention-distribution1}, the attention basin effect vanished entirely after the structure was removed. This result provides a profound insight: the phenomenon is not arbitrary. Instead, it reveals that the attention basin is a structural-level analogue to the well-known token-level primacy and recency effects. Just as models are known to over-attend to the first and last tokens of an entire sequence, our experiment demonstrates that they apply a similar heuristic at a higher level of abstraction, granting special status to the structural blocks positioned at the edges of the context. The model recognizes the collection of documents as a set and focuses its attention on the boundaries of that set.

\subsection{How Attention Influences LLMs' Performance}

The attention basin provides a mechanistic explanation for the widely observed lost-in-the-middle effect \cite{liu2024lost}. If documents at the edges receive more attention, it follows that their content would more strongly influence the model's output. We hypothesize that a document's contribution to the final answer is directly proportional to the attention it receives.

To formalize this intuition, we analyze the relationship between a document's attention and the model's output probability. Let $\bar{\alpha}_d = \frac{1}{L}\sum_{l=1}^L \alpha_d^{(l)}$ be the cross-layer average attention weight allocated to a document $d$, and let $P(y^*|\cdot)$ be the generation probability of the correct answer $y^*$. Under the simplifying assumption of semi-orthogonal document representations (Assumption~\ref{ass:ortho}), our theoretical analysis yields the following proposition, with the full proof in Appendix~\ref{main proof}.

\begin{proposition}[Attention-Probability Monotonicity]
\label{prop:mono-main}
For a correct document $d^*$ and any other document $d_j$, the partial derivatives of the correct answer's probability $P(y^*|\cdot)$ with respect to their average attention weights satisfy:
\begin{equation}
    \frac{\partial P(y^*|\cdot)}{\partial \bar{\alpha}_{d^*}} > \left|\frac{\partial P(y^*|\cdot)}{\partial \bar{\alpha}_{d_j}}\right| \ge 0.
\end{equation}
\end{proposition}

This proposition mathematically confirms that increasing attention on the correct document $d^*$ provides the most effective path to improving model performance. Corollary~\ref{cor:doc-pos} in the appendix further leverages this insight to explain the lost-in-the-middle phenomenon, demonstrating that placing crucial documents in high-attention regions (i.e., the edges) maximizes the probability of generating the correct answer.


We then validated this theory empirically using the HotpotQA dataset \cite{yang2018hotpotqa}. We constructed inputs with two ground-truth documents  ($d_1$, $d_2$) and one irrelevant noise document $n$ . We tested all six permutations of these documents, measuring both the attention distribution and the final QA accuracy for each. We categorized permutations into two groups: those where the ground-truth documents received the highest cumulative attention, and those where the noise document did.

\begin{figure*}[!h]
\centering
\includegraphics[width=1\linewidth]{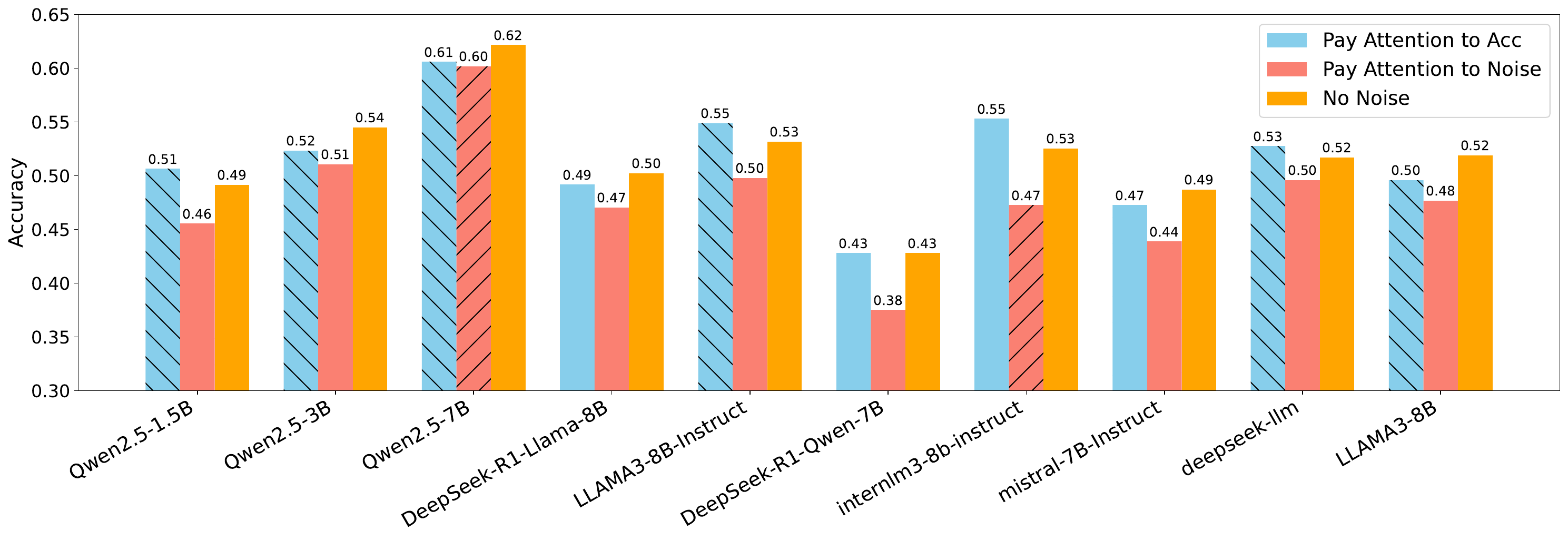}
\caption{Model QA accuracy on HotpotQA across all permutations of two relevant documents and one noise document. Blue bars: permutations where relevant documents receive the highest attention. Red bars: permutations where the noise document receives the highest attention. Orange bar: noise-free baseline. Aligning relevant documents with high-attention positions consistently yields the best performance.}
\label{fig:bar}
\end{figure*}

The results in Figure~\ref{fig:bar} are unequivocal. Permutations where the correct documents received the most attention (blue bars) significantly outperformed those where the noise document was attended to most (red bars). Remarkably, this attention-optimized ordering not only mitigated the impact of the distractor but, in some cases, even surpassed the performance of the noise-free baseline (orange bar). This demonstrates that controlling positional attention is a powerful mechanism for improving model robustness and accuracy.

\subsection{The Critical Role of Shallow Attention Layers}

Prior studies \cite{artzy2024attend,van2019does} have shown that attention across different layers contribute unequally to its behavior. Inspired by Abnar et al.~\cite{abnar2020quantifying}, who demonstrate that information becomes increasingly entangled in deeper layers, we analyze this effect by decomposing the attention score at a given position $p$ and layer $l$. To formalize the \textit{attention basin phenomenon}, our analysis is based on the assumption (detailed in Assumption~\ref{ass:pos-attn}) that expected attention can be decomposed into two parts: a deterministic, position-dependent bias $f(p)$ and a content-driven, position-agnostic component $\epsilon^{(l)}_p$. This is expressed as: $\mathbb{E}[\alpha^{(l)}_{p}] = f(p) + \epsilon^{(l)}_p$.

To quantify the balance between these two forces, we define a ratio $\rho(l)$ of their variances: $\rho(l) = \mathbb{V}[\epsilon^{(l)}_p] / \mathbb{V}[f(p)]$. This ratio leads to our second proposition.

\begin{proposition}[Layer-wise Attention Regimes]
\label{prop:layer-main}
There exists a layer depth threshold $L^*$ that partitions the model into two regimes based on the attention-type variance ratio $\rho(l)$:
\begin{enumerate}
    \item \textbf{Position-dominated regime} ($l < L^*$): $\rho(l) < 1$, where positional bias variance exceeds content variance.
    \item \textbf{Content-dominated regime} ($l \ge L^*$): $\rho(l) \ge 1$, where content-based attention variance becomes dominant.
\end{enumerate}
A full derivation is in Appendix~\ref{main proof} (Corollary~\ref{cor:layer}).
\end{proposition}

Thus, \textbf{shallow-layer attention distributions more accurately reflect the model's structural and positional focus} on different input segments. This insight is crucial, as it suggests that the attention patterns from early layers, where the positional signal $f(p)$ is strongest, are the most reliable signal for understanding and manipulating the model's positional preferences.

\begin{figure}[!t]
\centering
\includegraphics[width=1\linewidth]{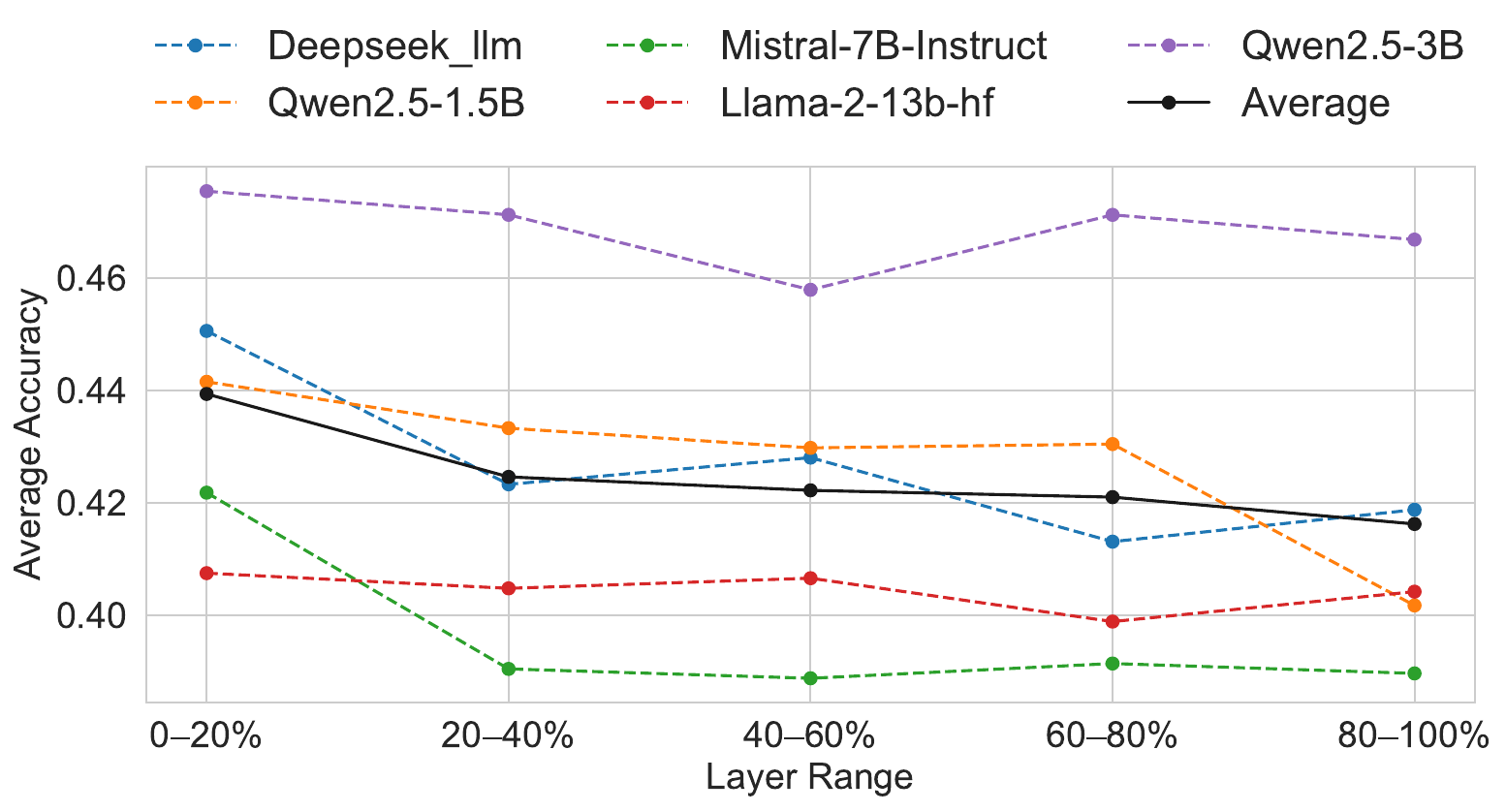}
\caption{Accuracy of a reranking strategy based on attention from different Transformer layers. Reranking using shallow-layer attention consistently outperforms using deeper layers, indicating that LLM's core positional bias is established early.}
\label{fig:layers}
\end{figure}

To verify this, we designed a reranking experiment. For a given query, we retrieved five relevant documents. We then used the attention scores from different layers of the model to determine the optimal position for the most important document. As shown in Figure~\ref{fig:layers}, reranking based on attention from the shallowest layers consistently yielded the highest QA accuracy. This confirms our hypothesis: the model's foundational positional bias is set early in the forward pass, making shallow-layer attention the most effective signal for understanding and ultimately mitigating this bias.
\section{Methodology}
\label{sec:method}

\begin{figure*}[!ht]
    \centering
    \includegraphics[width=1\textwidth]{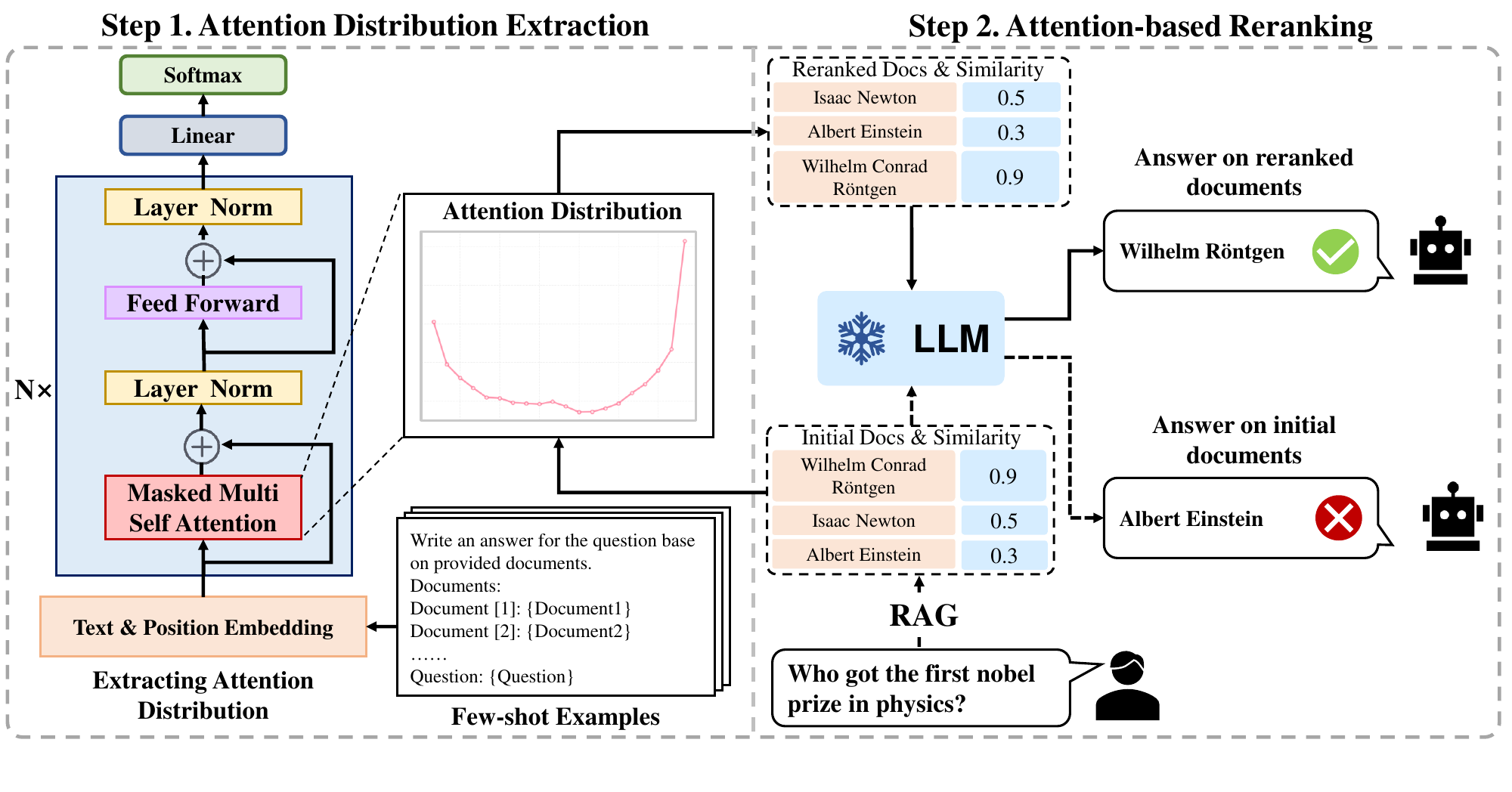}
    \caption{Overview of the AttnRank framework. Step 1: Profiling Positional Attention—We perform a one-time, low-cost analysis to capture the model's intrinsic attention pattern (the attention basin) across document positions. Step 2: Attention-ariven Reranking—For any new query, we reorder the retrieved documents, mapping the most relevant document (highest similarity) to the position with the highest profiled attention score, thus aligning relevance with the model's natural focus.}
    \label{fig:rerank_algorithm}
\end{figure*}

In the previous section, we demonstrated that an LLM's performance is highly sensitive to the positional ordering of documents, a phenomenon driven by the \textit{attention basin}. This finding suggests that instead of fighting against the model's intrinsic biases, we can strategically leverage them.

To this end, we propose the \textbf{Attention-Driven Reranker (\ours)}, a lightweight yet highly effective training-free method that aligns document relevance with the model's innate attention preferences. \ours intelligently reorders input documents to place the most critical information in the positions where the model is naturally inclined to focus, thereby mitigating positional bias and maximizing performance. As illustrated in Figure~\ref{fig:rerank_algorithm}, the framework operates in two key stages: \textbf{Attention Distribution Extraction} and \textbf{Attention-based Reranking}.

\subsection{Step 1: Attention Distribution Extraction}

The foundational step of \ours is to create a stable, general-purpose "attention profile" for a given LLM. This profile serves as a map of the model's positional biases. Based on our findings in Section~\ref{sec:observation}, we use the shallowest attention layer, as it provides the purest signal of the model's intrinsic positional preference.

To generate this profile, we craft a set of probe inputs $S_i = \{t, d_1, \ldots, d_k, q\}$, where $t$ is a fixed task template, $D = \{d_1, \ldots, d_k\}$ are placeholder documents, and $q$ is a query. We then compute the average attention paid by the query tokens to each document position across multiple samples:
\begin{equation}
\label{eq:attention_extraction}
\begin{split}
    A &= \{a_1, \ldots, a_k\} \\
    &= \frac{1}{N} \sum_{i=1}^{N} \mathrm{Attention}(\{d_1, \ldots, d_k\}_i \mid S_i),
\end{split}
\end{equation}
where $A$ is the final attention profile, $a_j$ is the average attention score for the $j$-th position, $N$ is the number of probe samples, and $\mathrm{Attention}(\cdot)$ extracts the mean attention from the query to each document block.

Crucially, we find that this profiling process is highly efficient. Our experiments show that a stable attention profile can be established with a remarkably small number of samples—often as few as 400. For some models, the characteristic \textit{attention basin} pattern emerges with just a single example (see Appendix~\ref{how many} for details). This one-time, low-cost profiling step yields a reusable attention map that captures the essence of the model's positional bias.

\subsection{Step 2: Attention-based Reranking}

Once the model's attention profile $A$ is established, it can be deployed to optimize document ordering for any subsequent task, such as Retrieval-Augmented Generation (RAG). The reranking procedure is as follows:

\begin{enumerate}[leftmargin=*]
    \item \textbf{Retrieve:} For a given user query, use a standard retriever to fetch the top-$k$ most relevant documents, $D = \{d_1, \ldots, d_k\}$, ranked by a similarity metric .
    
    \item \textbf{Rerank:} Instead of feeding the documents to the LLM in their default similarity-ranked order, we reorder them according to the pre-computed attention profile $A$. Specifically, we map the document with the highest relevance score ($d_1$) to the position with the highest attention score in $A$, the second-most relevant document ($d_2$) to the position with the second-highest attention, and so on.
    
    \item \textbf{Generate:} Concatenate the reordered documents with the query and prompt, and feed the final input into the LLM for generation.
\end{enumerate}

\begin{table*}[!h]
  \centering

  \resizebox{\textwidth}{!}{%
  \begin{tabular}{lccccc}
    \toprule
    \textbf{Models} & \textbf{Random} & \textbf{Descending} & \textbf{Ascending} & \textbf{LIM\cite{liu2024lost}} & \textbf{\ours (Ours)} \\
    \midrule
    InternLM3-8B-Instruct (\citeyear{cai2024internlm2}) & 41.41 & \textbf{42.92} & 40.91 & \underline{41.91} & 41.56 \\
    Mistral-7B-Instruct (\citeyear{2023arXiv231006825J})& 39.47 & 38.65 & 38.90 & \underline{41.52} & \textbf{42.17} \\
    LLAMA-2-13B-hf (\citeyear{touvron2023llama})& 38.37 & 38.20 & \underline{40.74} & 39.67 & \textbf{41.00} \\
    LLAMA3-8B (\citeyear{grattafiori2024llama}) & 41.36 & 40.76 & \textbf{43.19} & 42.09 & \underline{43.09} \\

    LLAMA3-8B-Instruct (\citeyear{grattafiori2024llama}) & 44.77 & \underline{45.04} & 44.81 & 44.02 & \textbf{46.32} \\
    DeepSeek-R1-Distill-Llama-8B (\citeyear{guo2025deepseek}) & 39.45 & 38.59 & 39.81 & \underline{39.82} & \textbf{41.77} \\
    DeepSeek-LLM (\citeyear{bi2024deepseek})& 42.14 & 41.20 & 41.51 & \underline{42.22} & \textbf{45.05} \\

    Qwen 2.5 1.5B (\citeyear{yang2024qwen2})& 40.62 & 41.23 & \underline{43.56} & 39.48 & \textbf{44.14} \\
    Qwen 2.5 3B (\citeyear{yang2024qwen2})& 45.77 & 46.42 & \underline{47.45} & 45.62 & \textbf{47.54} \\
    Qwen 2.5 7B (\citeyear{yang2024qwen2})& 52.32 & 53.31 & \textbf{54.64} & 52.18 & \underline{54.55} \\

    \midrule
    \textbf{Average Accuracy} & 42.57 & 42.63 & \underline{43.55} & 42.85 & \textbf{44.72} \\
    \bottomrule
  \end{tabular}%
  }
    \caption{Answer accuracy (\%) on the HotpotQA dataset using different document ordering strategies. Five documents are retrieved per question, and models are evaluated on how accurately they answer the question based on the ordered input.}
  \label{tab:hotpotqa_results}
\end{table*}

By synchronizing the relevance hierarchy of the documents with the attention hierarchy of the model, \ours ensures that the most critical information is placed exactly where the model is hardwired to look. This alignment preemptively resolves the conflict between document importance and positional bias, allowing the model to focus its computational resources effectively and avoid distractions from less relevant information.

 As a training-free and model-agnostic method, \ours can be universally applied to any LLM without altering its parameters. The computational overhead is negligible, consisting of a single, reusable profiling step. Furthermore, because \ours operates as an input pre-processing stage, it is fully compatible with modern inference acceleration frameworks like Flash Attention and vLLM, allowing its benefits to be seamlessly integrated without compromising existing performance optimizations.

\section{Experiments}
\label{sec:exps}

To validate the effectiveness of \ours, we raise the following three key Research Questions (\textbf{RQs}). We conduct comprehensive experiments on various datasets using 10 mainstream LLMs across different architectures and scales to systematically address these \textbf{RQs}:
\begin{itemize}[leftmargin=*]
  \item \textbf{RQ1}: Does placing informative documents in high‐attention slots improve LLM's performance?
  \item \textbf{RQ2:} How does the effectiveness of \ours generalize across a diverse range of LLM architectures and scales?
  \item \textbf{RQ3:} Is \ours robust across different tasks and datasets that rely on integrating multiple documents?
\end{itemize}
\subsection{Experimental setup}
\paragraph{Utilized LLMs}To assess the generality of attention basin and the broad applicability of \ours, we conduct experiments on a diverse set of models with varying structures and parameter counts, including \textbf{\textit{DeepSeek}} series (7B, Distill-Llama), \textbf{\textit{Llama}} series (7B, 13B), MoE-based \textbf{\textit{Mistral-7B}}, \textbf{\textit{Qwen2.5}} series (1.5B-7B) and \textbf{\textit{InternLM3-8B}}, covering most of the mainstream advanced LLMs, details of the models are provided in Appendix~\ref{LLMs}.

\paragraph{Baselines}Four reranking baselines are applied to the retrieved document set: (1) \textbf{Random}: Documents shuffled randomly. (2) \textbf{Similarity Descending}: Documents sorted by retrieval similarity in descending order.
(3) \textbf{Similarity Ascending}: Documents sorted by retrieval similarity in ascending order.
(4) \textbf{Lost-in-the-Middle (LIM)} \cite{liu2024lost}: Places the highest-similarity documents at the beginning and end of the input sequence.


\paragraph{Datasets} 
We evaluate the effectiveness of \ours on three widely-used benchmark datasets. For multi-hop question answering, we use \textbf{HotpotQA} \cite{yang2018hotpotqa} and \textbf{2WikiMultiHopQA} \cite{ho2020constructing} to assess complex reasoning capabilities. For multi-domain conversational tasks, we employ the \textbf{MultiWOZ} \cite{eric2020multiwoz,ye2022multiwoz} series dataset. Details for each dataset are provided in Appendix~\ref{LLMs}.

\paragraph{Metrics}
We evaluate multi-hop QA performance using \textbf{answer average accuracy} as a straightforward metric.
For few-shot experiment, we use the \textbf{Joint Goal Accuracy (JGA)} \cite{henderson2014word} as the evaluation metric. For each turn in a dialogue, a dialogue state is
considered correct only if it \textbf{exactly matches} the ground truth.

\subsection{Multi-hop QA experiment}
\label{exp1}
Multi-hop QA serves as an ideal testbed for our hypothesis, as it forces the model to locate and integrate key facts scattered across several documents. Success is highly dependent on the model's ability to process the provided context effectively.

\paragraph{Experimental setup} 
For each question, five candidate documents are retrieved by beam-retriever \cite{zhang2024end}, which achieves state-of-the-art retrieval performance on retrieving tasks. As shown in Appendix~\ref{sec:appendix-exps-implement}), all reranked document lists are fed into a unified QA model with a identical prompt template.

\begin{table*}[!h]
  \centering
\setlength{\tabcolsep}{3.5pt} 

  \begin{tabular}{lccccc}
    \toprule
    \textbf{Models} & \textbf{Random} & \textbf{Descending} & \textbf{Ascending} & \textbf{LIM \cite{liu2024lost}} & \textbf{\ours (Ours)} \\
    \midrule
    InternLM3-8B-Instruct (\citeyear{cai2024internlm2})& 39.14 & \textbf{40.87} & 38.41 & \underline{39.94} & 39.56 \\
    Mistral-7B-Instruct (\citeyear{2023arXiv231006825J})& 29.64 & 28.47 & \textbf{33.14} & 27.74 & \underline{32.54} \\
    LLAMA-2-13B-hf (\citeyear{touvron2023llama})& 30.21 & 30.27 & \textbf{31.79} & 30.18 & \textbf{31.79} \\
    LLAMA3-8B (\citeyear{grattafiori2024llama})& 32.17 & 30.78 & \underline{33.39} & 30.35 & \textbf{33.97} \\
    LLAMA3-8B-Instruct (\citeyear{grattafiori2024llama})& 37.39 & 37.91 & \textbf{39.35} & 36.59 & \underline{38.14} \\
    DeepSeek-R1-Distill-Llama-8B (\citeyear{guo2025deepseek}) & 27.09 & 26.68 & \underline{28.02} & 26.78 & \textbf{28.24} \\

    DeepSeek-LLM (\citeyear{bi2024deepseek})& 27.83 & 28.12 & \underline{30.54} & 28.16 & \textbf{31.31} \\

    Qwen 2.5 1.5B (\citeyear{yang2024qwen2})& 30.05 & 29.48 & \underline{33.34} & 29.35 & \textbf{33.97} \\
    Qwen 2.5 3B (\citeyear{yang2024qwen2})& 32.72 & \textbf{34.33} & 33.55 & 32.18 & \underline{34.08} \\
    Qwen 2.5 7B (\citeyear{yang2024qwen2})& 41.22 & 41.55 & \underline{43.44} & 39.71 & \textbf{43.55} \\
    \midrule
    \textbf{Average} & 32.75 & 32.85 & \underline{34.50} & 32.10 & \textbf{34.72} \\
    \bottomrule
  \end{tabular}%
    \caption{Answer accuracy (\%) on 2WikiMultiHopQA using different document ordering strategies. Five documents are retrieved per question, and models are evaluated on answer correctness.}
  \label{tab:2wiki_results}
\end{table*}

\paragraph{Experiment results and analyze} 

As shown in Tables~\ref{tab:hotpotqa_results} and~\ref{tab:2wiki_results}, \ours{} consistently outperforms baselines on HotpotQA (44.72\% vs.\ Random 42.57\%, Descending 42.63\%, Ascending 43.55\%, LIM 42.85\%) and 2WikiMultiHopQA (34.72\% vs.\ Random 32.75\%, Descending 32.85\%, Ascending 34.50\%, LIM 32.10\%). Notably, while ascending order occasionally performs competitively, particularly on 2WikiMultiHopQA, \ours still consistently yields equal or superior accuracy, indicating that simple heuristics do not fully capture the model's attention dynamics. These results answer \textbf{RQ1} that placing the most informative documents in positions where the model’s attention is most focused significantly enhances its ability to reason across multiple hops.  The consistent improvements across different architectures (e.g., LLAMA3, Qwen, Mistral, DeepSeek) and model sizes ranging from 1.5B to 13B further validate the generalizability of \ours, thereby answering \textbf{RQ2}. Additional experimental details, including attention distribution and case studies, are provided in Appendix~\ref{add exp}.

\subsection{Few-shot experiment}
Few-shot generation task involves generating outputs based on a small number of example demonstrations and a user query. We design an experiment to assess the \ours in this setting.
The IC-DST algorithm \cite{hu2022context} retrieves a small set of in-context examples to generate SQL queries, thereby extracting user intent from dialogue history and maintaining an up-to-date dialogue state. The precision of the dialogue state is strongly correlated with the quality of the code generation, for which we choose to test the effectiveness of the \ours based on the IC-DST algorithm.

\paragraph{Experimental setup} 
For each dialogue turn, the trained IC-DST retriever fetch $k$ example dialogues from the few-shot context pool. The same five reranking strategies described in Section \ref{exp1} are applied to the retrieved examples. The sorted examples and the current dialogue history are then passed to the Code Llama \cite{roziere2023code} with a fixed prompt template (see Appendix~\ref{sec:appendix-exps-implement}).

\paragraph{Experiment results and analyze} 

\begin{table}[!h]
\setlength{\tabcolsep}{3.5pt}
  \centering
  \begin{tabular}{lccc}
    \toprule
    \multirow{2}{*}{Method} & MultiWOZ & MultiWOZ & \multirow{2}{*}{Average} \\
                            & 2.1        & 2.4        &                  \\
    \midrule
    Random                  & 41.12      & 50.11      & 45.62            \\
    Descending              & 42.10      & 50.61      & 46.36            \\
    Ascending               & 42.67      & 51.16      & 46.92            \\
    LIM\cite{liu2024lost}                     & 42.44      & 51.14      & 46.79            \\
    \ours \textbf{(Ours)}   & \textbf{42.89} & \textbf{51.51} & \textbf{47.20} \\
    \bottomrule
  \end{tabular}%
  \caption{Joint Goal Accuracy (\%) on MultiWOZ 2.1 and 2.4 (1\% few-shot sample), under different context retrieval reranking methods.}
  \label{tab:fewshot}
\end{table}

As shown in Table~\ref{tab:fewshot}, \ours outperforms all baselines on MultiWOZ 2.1 and 2.4, improving over random by 1.58\%. All structured ordering strategies (Descending, Ascending, LIM) yield gains over random, indicating that systematic example ordering enhances context relevance for dialogue state tracking. \ours’s additional improvements over Ascending and LIM confirm that aligning high-value examples with the model’s attention peaks further boosts extraction accuracy, validating our hypothesis.

Overall, the strong performance on both multi-hop QA and few-shot learning—spanning four datasets and two distinct task types—provides a comprehensive and positive answer to \textbf{RQ3}. Our findings demonstrate that \ours is a robust and effective method for mitigating negative positional effects across diverse scenarios.
\section{Conclusion}
\label{conclusion}


In this work, we identified a fundamental mechanism governing positional bias in Large Language Models: the attention basin phenomenon. We demonstrated that LLMs exhibit a systematic tendency to focus on the beginning and end of a structured input block, a predictable pattern rather than a random flaw. This core insight allowed us to propose Attention-Driven Reranking (ADR), a lightweight and training-free framework that transforms this bias from a liability into an asset. By strategically reordering input items to align salient information with the model's natural attention peaks, ADR effectively enhances knowledge utilization. Crucially, as a model-agnostic and parameter-free method, it requires no architectural modifications and seamlessly integrates with existing pipelines. This makes it fully compatible with modern acceleration frameworks, offering a rare combination of improved accuracy and high efficiency. We believe this principle of "attention alignment" opens new avenues for research, and we discuss potential limitations and future directions in Appendix~\ref{limitation} to inspire further exploration.

\bibliography{AttnRank}

\begin{thebibliography}{35}
\providecommand{\natexlab}[1]{#1}

\bibitem[{Abnar and Zuidema(2020)}]{abnar2020quantifying}
Abnar, S.; and Zuidema, W. 2020.
\newblock Quantifying Attention Flow in Transformers.
\newblock In \emph{Proceedings of the 58th Annual Meeting of the Association for Computational Linguistics}, 4190--4197.

\bibitem[{An et~al.(2024)An, Ma, Lin, Zheng, Lou, and Chen}]{an2024make}
An, S.; Ma, Z.; Lin, Z.; Zheng, N.; Lou, J.-G.; and Chen, W. 2024.
\newblock Make your llm fully utilize the context.
\newblock \emph{Advances in Neural Information Processing Systems}, 37: 62160--62188.

\bibitem[{Artzy and Schwartz(2024)}]{artzy2024attend}
Artzy, A.~B.; and Schwartz, R. 2024.
\newblock Attend First, Consolidate Later: On the Importance of Attention in Different {LLM} Layers.
\newblock In Belinkov, Y.; Kim, N.; Jumelet, J.; Mohebbi, H.; Mueller, A.; and Chen, H., eds., \emph{Proceedings of the 7th BlackboxNLP Workshop: Analyzing and Interpreting Neural Networks for NLP}, 177--184. Miami, Florida, US: Association for Computational Linguistics.

\bibitem[{Bi et~al.(2024)Bi, Chen, Chen, Chen, Dai, Deng, Ding, Dong, Du, Fu et~al.}]{bi2024deepseek}
Bi, X.; Chen, D.; Chen, G.; Chen, S.; Dai, D.; Deng, C.; Ding, H.; Dong, K.; Du, Q.; Fu, Z.; et~al. 2024.
\newblock Deepseek llm: Scaling open-source language models with longtermism.
\newblock \emph{arXiv preprint arXiv:2401.02954}.

\bibitem[{Brown et~al.(2020)Brown, Mann, Ryder, Subbiah, Kaplan, Dhariwal, Neelakantan, Shyam, Sastry, Askell et~al.}]{brown2020language}
Brown, T.; Mann, B.; Ryder, N.; Subbiah, M.; Kaplan, J.~D.; Dhariwal, P.; Neelakantan, A.; Shyam, P.; Sastry, G.; Askell, A.; et~al. 2020.
\newblock Language models are few-shot learners.
\newblock \emph{Advances in neural information processing systems}, 33: 1877--1901.

\bibitem[{Cai et~al.(2024)Cai, Cao, Chen, Chen, Chen, Chen, Chen, Chen, Chen, Chu et~al.}]{cai2024internlm2}
Cai, Z.; Cao, M.; Chen, H.; Chen, K.; Chen, K.; Chen, X.; Chen, X.; Chen, Z.; Chen, Z.; Chu, P.; et~al. 2024.
\newblock Internlm2 technical report.
\newblock \emph{arXiv preprint arXiv:2403.17297}.

\bibitem[{Eric et~al.(2020)Eric, Goel, Paul, Sethi, Agarwal, Gao, Kumar, Goyal, Ku, and Hakkani-Tur}]{eric2020multiwoz}
Eric, M.; Goel, R.; Paul, S.; Sethi, A.; Agarwal, S.; Gao, S.; Kumar, A.; Goyal, A.; Ku, P.; and Hakkani-Tur, D. 2020.
\newblock MultiWOZ 2.1: A Consolidated Multi-Domain Dialogue Dataset with State Corrections and State Tracking Baselines.
\newblock In \emph{Proceedings of the Twelfth Language Resources and Evaluation Conference}, 422--428.

\bibitem[{Grattafiori et~al.(2024)Grattafiori, Dubey, Jauhri, Pandey, Kadian, Al-Dahle, Letman, Mathur, Schelten, Vaughan et~al.}]{grattafiori2024llama}
Grattafiori, A.; Dubey, A.; Jauhri, A.; Pandey, A.; Kadian, A.; Al-Dahle, A.; Letman, A.; Mathur, A.; Schelten, A.; Vaughan, A.; et~al. 2024.
\newblock The llama 3 herd of models.
\newblock \emph{arXiv preprint arXiv:2407.21783}.

\bibitem[{Guo et~al.(2025)Guo, Yang, Zhang, Song, Zhang, Xu, Zhu, Ma, Wang, Bi et~al.}]{guo2025deepseek}
Guo, D.; Yang, D.; Zhang, H.; Song, J.; Zhang, R.; Xu, R.; Zhu, Q.; Ma, S.; Wang, P.; Bi, X.; et~al. 2025.
\newblock Deepseek-r1: Incentivizing reasoning capability in llms via reinforcement learning.
\newblock \emph{arXiv preprint arXiv:2501.12948}.

\bibitem[{{Guo} et~al.(2025){Guo}, {Zeng}, {Zhao}, {Liu}, {Yu}, {Du}, {Chen}, and {Cheng}}]{2025arXiv250722050G}
{Guo}, M.; {Zeng}, Q.; {Zhao}, X.; {Liu}, Y.; {Yu}, W.; {Du}, M.; {Chen}, H.; and {Cheng}, W. 2025.
\newblock {DeepSieve: Information Sieving via LLM-as-a-Knowledge-Router}.
\newblock \emph{arXiv e-prints}, arXiv:2507.22050.

\bibitem[{Henderson, Thomson, and Young(2014)}]{henderson2014word}
Henderson, M.; Thomson, B.; and Young, S. 2014.
\newblock Word-based dialog state tracking with recurrent neural networks.
\newblock In \emph{Proceedings of the 15th annual meeting of the special interest group on discourse and dialogue (SIGDIAL)}, 292--299.

\bibitem[{Ho et~al.(2020)Ho, Nguyen, Sugawara, and Aizawa}]{ho2020constructing}
Ho, X.; Nguyen, A.-K.~D.; Sugawara, S.; and Aizawa, A. 2020.
\newblock Constructing A Multi-hop QA Dataset for Comprehensive Evaluation of Reasoning Steps.
\newblock In \emph{Proceedings of the 28th International Conference on Computational Linguistics}, 6609--6625.

\bibitem[{Hu et~al.(2022)Hu, Lee, Xie, Yu, Smith, and Ostendorf}]{hu2022context}
Hu, Y.; Lee, C.-H.; Xie, T.; Yu, T.; Smith, N.~A.; and Ostendorf, M. 2022.
\newblock In-Context Learning for Few-Shot Dialogue State Tracking.
\newblock In \emph{Findings of the Association for Computational Linguistics: EMNLP 2022}, 2627--2643.

\bibitem[{{Jiang} et~al.(2023){Jiang}, {Sablayrolles}, {Mensch}, {Bamford}, {Singh Chaplot}, {de las Casas}, {Bressand}, {Lengyel}, {Lample}, {Saulnier}, {Renard Lavaud}, {Lachaux}, {Stock}, {Le Scao}, {Lavril}, {Wang}, {Lacroix}, and {El Sayed}}]{2023arXiv231006825J}
{Jiang}, A.~Q.; {Sablayrolles}, A.; {Mensch}, A.; {Bamford}, C.; {Singh Chaplot}, D.; {de las Casas}, D.; {Bressand}, F.; {Lengyel}, G.; {Lample}, G.; {Saulnier}, L.; {Renard Lavaud}, L.; {Lachaux}, M.-A.; {Stock}, P.; {Le Scao}, T.; {Lavril}, T.; {Wang}, T.; {Lacroix}, T.; and {El Sayed}, W. 2023.
\newblock {Mistral 7B}.
\newblock \emph{arXiv e-prints}, arXiv:2310.06825.

\bibitem[{Kwon et~al.(2023)Kwon, Li, Zhuang, Sheng, Zheng, Yu, Gonzalez, Zhang, and Stoica}]{kwon2023efficient}
Kwon, W.; Li, Z.; Zhuang, S.; Sheng, Y.; Zheng, L.; Yu, C.~H.; Gonzalez, J.; Zhang, H.; and Stoica, I. 2023.
\newblock Efficient memory management for large language model serving with pagedattention.
\newblock In \emph{Proceedings of the 29th Symposium on Operating Systems Principles}, 611--626.

\bibitem[{Lewis et~al.(2020)Lewis, Perez, Piktus, Petroni, Karpukhin, Goyal, K{\"u}ttler, Lewis, Yih, Rockt{\"a}schel et~al.}]{lewis2020retrieval}
Lewis, P.; Perez, E.; Piktus, A.; Petroni, F.; Karpukhin, V.; Goyal, N.; K{\"u}ttler, H.; Lewis, M.; Yih, W.-t.; Rockt{\"a}schel, T.; et~al. 2020.
\newblock Retrieval-augmented generation for knowledge-intensive nlp tasks.
\newblock \emph{Advances in neural information processing systems}, 33: 9459--9474.

\bibitem[{Li et~al.(2023)Li, Shao, Xie, Sheng, Zheng, Gonzalez, Stoica, Ma, and Zhang}]{li2023long}
Li, D.; Shao, R.; Xie, A.; Sheng, Y.; Zheng, L.; Gonzalez, J.; Stoica, I.; Ma, X.; and Zhang, H. 2023.
\newblock How long can context length of open-source llms truly promise?
\newblock In \emph{NeurIPS 2023 Workshop on Instruction Tuning and Instruction Following}.

\bibitem[{Li et~al.(2024)Li, Wang, Tran, Xia, and Du}]{li2024meqa}
Li, R.; Wang, Z.; Tran, S.; Xia, L.; and Du, X. 2024.
\newblock MEQA: A Benchmark for Multi-hop Event-centric Question Answering with Explanations.
\newblock \emph{Advances in Neural Information Processing Systems}, 37: 126835--126862.

\bibitem[{Ling et~al.(2025)Ling, Xie, Dong, and Shen}]{ling2025enhancing}
Ling, Z.; Xie, Y.; Dong, C.; and Shen, Y. 2025.
\newblock Enhancing Factual Consistency in Text Summarization via Counterfactual Debiasing.
\newblock In \emph{Proceedings of the 31st International Conference on Computational Linguistics}, 7912--7924.

\bibitem[{Liu et~al.(2024)Liu, Lin, Hewitt, Paranjape, Bevilacqua, Petroni, and Liang}]{liu2024lost}
Liu, N.~F.; Lin, K.; Hewitt, J.; Paranjape, A.; Bevilacqua, M.; Petroni, F.; and Liang, P. 2024.
\newblock Lost in the middle: How language models use long contexts.
\newblock \emph{Transactions of the Association for Computational Linguistics}, 12: 157--173.

\bibitem[{Ma et~al.(2024)Ma, Chen, Zhang, Miao, Zhu, Chen, Xu, Li, Fan, Pan et~al.}]{ma2024compressing}
Ma, D.; Chen, L.; Zhang, S.; Miao, Y.; Zhu, S.; Chen, Z.; Xu, H.; Li, H.; Fan, S.; Pan, L.; et~al. 2024.
\newblock Compressing KV Cache for Long-Context LLM Inference with Inter-Layer Attention Similarity.
\newblock \emph{arXiv preprint arXiv:2412.02252}.

\bibitem[{Roziere et~al.(2023)Roziere, Gehring, Gloeckle, Sootla, Gat, Tan, Adi, Liu, Sauvestre, Remez et~al.}]{roziere2023code}
Roziere, B.; Gehring, J.; Gloeckle, F.; Sootla, S.; Gat, I.; Tan, X.~E.; Adi, Y.; Liu, J.; Sauvestre, R.; Remez, T.; et~al. 2023.
\newblock Code llama: Open foundation models for code.
\newblock \emph{arXiv preprint arXiv:2308.12950}.

\bibitem[{Touvron et~al.(2023)Touvron, Martin, Stone, Albert, Almahairi, Babaei, Bashlykov, Batra, Bhargava, Bhosale et~al.}]{touvron2023llama}
Touvron, H.; Martin, L.; Stone, K.; Albert, P.; Almahairi, A.; Babaei, Y.; Bashlykov, N.; Batra, S.; Bhargava, P.; Bhosale, S.; et~al. 2023.
\newblock Llama 2: Open foundation and fine-tuned chat models.
\newblock \emph{arXiv preprint arXiv:2307.09288}.

\bibitem[{Van~Aken et~al.(2019)Van~Aken, Winter, L{\"o}ser, and Gers}]{van2019does}
Van~Aken, B.; Winter, B.; L{\"o}ser, A.; and Gers, F.~A. 2019.
\newblock How does bert answer questions? a layer-wise analysis of transformer representations.
\newblock In \emph{Proceedings of the 28th ACM international conference on information and knowledge management}, 1823--1832.

\bibitem[{Vaswani et~al.(2017)Vaswani, Shazeer, Parmar, Uszkoreit, Jones, Gomez, Kaiser, and Polosukhin}]{vaswani2017attention}
Vaswani, A.; Shazeer, N.; Parmar, N.; Uszkoreit, J.; Jones, L.; Gomez, A.~N.; Kaiser, {\L}.; and Polosukhin, I. 2017.
\newblock Attention is all you need.
\newblock \emph{Advances in Neural Information Processing Systems}, 30.

\bibitem[{Wang et~al.(2025)Wang, Ding, Chen, Wu, Wang, Xie, and Zhao}]{wang2025vidorag}
Wang, Q.; Ding, R.; Chen, Z.; Wu, W.; Wang, S.; Xie, P.; and Zhao, F. 2025.
\newblock Vidorag: Visual document retrieval-augmented generation via dynamic iterative reasoning agents.
\newblock \emph{arXiv preprint arXiv:2502.18017}.

\bibitem[{Xiao et~al.(2024)Xiao, Tian, Chen, Han, and Lewis}]{xiao2023efficient}
Xiao, G.; Tian, Y.; Chen, B.; Han, S.; and Lewis, M. 2024.
\newblock Efficient Streaming Language Models with Attention Sinks.
\newblock In \emph{The Twelfth International Conference on Learning Representations}.

\bibitem[{Yang et~al.(2024)Yang, Yang, Zhang, Hui, Zheng, Yu, Li, Liu, Huang, Wei et~al.}]{yang2024qwen2}
Yang, A.; Yang, B.; Zhang, B.; Hui, B.; Zheng, B.; Yu, B.; Li, C.; Liu, D.; Huang, F.; Wei, H.; et~al. 2024.
\newblock Qwen2. 5 technical report.
\newblock \emph{arXiv preprint arXiv:2412.15115}.

\bibitem[{Yang et~al.(2018)Yang, Qi, Zhang, Bengio, Cohen, Salakhutdinov, and Manning}]{yang2018hotpotqa}
Yang, Z.; Qi, P.; Zhang, S.; Bengio, Y.; Cohen, W.; Salakhutdinov, R.; and Manning, C.~D. 2018.
\newblock HotpotQA: A Dataset for Diverse, Explainable Multi-hop Question Answering.
\newblock In \emph{Proceedings of the 2018 Conference on Empirical Methods in Natural Language Processing}, 2369--2380.

\bibitem[{Ye, Manotumruksa, and Yilmaz(2022)}]{ye2022multiwoz}
Ye, F.; Manotumruksa, J.; and Yilmaz, E. 2022.
\newblock MultiWOZ 2.4: A Multi-Domain Task-Oriented Dialogue Dataset with Essential Annotation Corrections to Improve State Tracking Evaluation.
\newblock In \emph{Proceedings of the 23rd Annual Meeting of the Special Interest Group on Discourse and Dialogue}, 351--360.

\bibitem[{Yi et~al.(2024)Yi, Ouyang, Liu, Liao, Xu, and Shen}]{yi2024survey}
Yi, Z.; Ouyang, J.; Liu, Y.; Liao, T.; Xu, Z.; and Shen, Y. 2024.
\newblock A survey on recent advances in llm-based multi-turn dialogue systems.
\newblock \emph{arXiv preprint arXiv:2402.18013}.

\bibitem[{Yi, Xu, and Shen(2025)}]{yi2025intent}
Yi, Z.; Xu, Z.; and Shen, Y. 2025.
\newblock Intent-driven In-context Learning for Few-shot Dialogue State Tracking.
\newblock In \emph{ICASSP 2025-2025 IEEE International Conference on Acoustics, Speech and Signal Processing (ICASSP)}, 1--5. IEEE.

\bibitem[{Zhang et~al.(2024)Zhang, Zhang, Zhang, Yong, and Huang}]{zhang2024end}
Zhang, J.; Zhang, H.; Zhang, D.; Yong, L.; and Huang, S. 2024.
\newblock End-to-End Beam Retrieval for Multi-Hop Question Answering.
\newblock In \emph{Proceedings of the 2024 Conference of the North American Chapter of the Association for Computational Linguistics: Human Language Technologies (Volume 1: Long Papers)}, 1718--1731.

\bibitem[{Zhang et~al.(2025)Zhang, Li, Long, Zhang, Lin, Yang, Xie, Yang, Liu, Lin et~al.}]{zhang2025qwen3}
Zhang, Y.; Li, M.; Long, D.; Zhang, X.; Lin, H.; Yang, B.; Xie, P.; Yang, A.; Liu, D.; Lin, J.; et~al. 2025.
\newblock Qwen3 Embedding: Advancing Text Embedding and Reranking Through Foundation Models.
\newblock \emph{arXiv preprint arXiv:2506.05176}.

\bibitem[{Zhang, Wen, and Zhao(2025)}]{zhang2025rule}
Zhang, Z.; Wen, L.; and Zhao, W. 2025.
\newblock Rule-KBQA: Rule-Guided Reasoning for Complex Knowledge Base Question Answering with Large Language Models.
\newblock In \emph{Proceedings of the 31st International Conference on Computational Linguistics}, 8399--8417.

\end{thebibliography}

\clearpage
\appendix
\section{Technical Appendix}
\label{appendix}

The appendix is organized as follows:
\begin{itemize}[leftmargin=*]
  \item \textbf{Sec.~\ref{sec:appendix-exps-implement}: Implementation of experiments}.  
  
  \item \textbf{Sec.~\ref{sec:appendix-more-attn}: Multi‑document attention distributions}.  
    Attention distributions across models with varying numbers of input documents.

  \item \textbf{Sec.~\ref{main proof}: Theoretical analysis of attention‑guided reranking}, comprising:
    \begin{itemize}[leftmargin=1em]
      \item Sec. \ref{e1} Optimization objective and problem statement  
      \item Sec. \ref{e2} Key assumptions
      \item Sec. \ref{e3} Technical lemmas on hidden‑state decomposition and logit formation  
      \item Sec. \ref{e4} Main proposition
      \item Sec. \ref{e5} Corollaries on optimal document positioning and layer‑wise effects  
      \item Sec. \ref{e6} Connection between our theory and \ours

    \end{itemize}


  \item \textbf{Sec.~\ref{how many}: Data volume requirements}.  
    Convergence analysis of data requirements for reliably characterizing LLM attention distribution patterns.
\item \textbf{Sec.~\ref{add exp}: Supplementary experiments and case studies.}
 \item \textbf{Sec.~\ref{limitation}: Limitation and future works}. 

\end{itemize}

\section{Implementation of experiments}
\label{sec:appendix-exps-implement}

\paragraph{Advantages of Our Method}
The \ours framework is designed to be both effective and practical, offering several key advantages:

\begin{itemize}
    \item \textbf{Training-Free and Model-Agnostic:} \ours requires no modification to the LLM's architecture or parameters. It treats the model as a black box, making it universally applicable to any Transformer-based LLM.

    \item \textbf{Extremely Low Overhead:} The primary cost is a \textbf{one-time profiling step}, which, as we've shown, requires a minimal number of inference runs. Once the attention profile is saved, it can be reused for all future inference tasks for that model at virtually no cost. The reranking itself is a simple array permutation, which is computationally negligible.

    \item \textbf{Compatibility with Modern Acceleration Frameworks:} Because \ours is an input-preprocessing step that operates before the main generation pass, it is fully compatible with popular inference acceleration libraries like \textbf{Flash Attention} and serving frameworks like \textbf{vLLM}. It does not interfere with their internal optimizations, allowing users to gain performance from our method while retaining the benefits of these high-speed tools.
\end{itemize}

\paragraph{Experimental setup}
All experiments were conducted on the same hardware configuration, as detailed in Table~\ref{tab:env_config}. In the first stage of \ours, we use the \texttt{AutoModelForCausalLM} function from the HuggingFace Transformers library to load the model, perform inference, and extract attention distributions. In the second stage, we adopt the VLLM framework \cite{kwon2023efficient} to accelerate inference.

\begin{table}[htbp]
  \centering
  
  \resizebox{\linewidth}{!}{%
  \begin{tabular}{@{} cc @{}} 
    \toprule
    \textbf{Architecture}           & x86\_64                               \\
    \midrule
    \textbf{CPU}                    & Intel Xeon Gold 5218R @ 2.10GHz       \\
    \textbf{GPU}                    & NVIDIA GeForce RTX 3090 24GB $\times$ 10 \\
    \textbf{CUDA Toolkit}           & 11.3                                  \\
    \textbf{Operating System}       & Ubuntu 20.04                          \\
    \textbf{Programming Language}   & Python 3.9.18                         \\
    \textbf{Deep Learning Framework} & PyTorch 1.13.0                        \\
    \bottomrule
  \end{tabular}%
  }
  \caption{Experimental Environment Configuration}
  \label{tab:env_config}
\end{table}

\paragraph{Utilized LLMs}
\label{LLMs}
\begin{itemize}[leftmargin=*]
  \item \textbf{DeepSeek-R1-Distill-Llama-8B} \cite{guo2025deepseek}: An 8B-parameter model distilled from DeepSeek-R1, fine-tuned via reinforcement learning to enhance reasoning capabilities.

  \item \textbf{DeepSeek-LLM} \cite{bi2024deepseek}: A 7B parameter models trained on 2 trillion tokens, utilizing a pre-norm decoder-only Transformer architecture with grouped-query attention (GQA).

  \item \textbf{LLaMA-3 Series} \cite{grattafiori2024llama}: Meta's LLaMA 3 series includes 8B and 70B parameter models trained on 15 trillion tokens.The instruct variant is fine-tuned to enhance instruction-following capabilities.

  \item \textbf{LLaMA-2-13B-hf} \cite{touvron2023llama}: A 13B parameter decoder-only Transformer model, trained on 2 trillion tokens, featuring a 4,096-token context window optimized for general-purpose language tasks.

  \item \textbf{Mistral-7B-Instruct} \cite{2023arXiv231006825J}: A 7B parameter model fine-tuned for instruction following, employing grouped-query attention and sliding window attention mechanisms to handle long sequences.

  \item \textbf{InternLM3-8B-Instruct} \cite{cai2024internlm2}: An 8B parameter model fine-tuned to follow instructions, designed to perform effectively across a variety of natural language understanding tasks.

  \item \textbf{Qwen 2.5 Series} \cite{yang2024qwen2}: Developed by Alibaba, the Qwen 2.5 series includes models with 1.5B, 3B, and 7B parameters, trained on extensive datasets to support multilingual capabilities and demonstrate strong performance across diverse tasks.

  \item \textbf{Code LLaMA 7B} \cite{roziere2023code}: A code-specialized model fine-tuned from LLaMA 2, supporting code generation tasks across multiple programming languages.
\end{itemize}

\paragraph{Utilized Datasets}

\begin{itemize} 
    \item \textbf{HotpotQA} \cite{yang2018hotpotqa} is a large-scale dataset designed to facilitate research in multi-hop question answering. It comprises approximately 113,000 question-answer pairs, each requiring reasoning over multiple Wikipedia articles to derive the correct answer.
    \item \textbf{2WikiMultiHopQA} \cite{ho2020constructing} is a dataset constructed to assess the comprehensive reasoning abilities of question answering models. It contains question-answer pairs that require multi-hop reasoning over both structured data from Wikidata and unstructured text from Wikipedia. 
    \item \textbf{MultiWOZ} \cite{eric2020multiwoz,ye2022multiwoz} series are multi-domain task-oriented dialogue datasets with annotated dialogue states. Each dialogue covers 1 to 5 domains, with an average of 13.7 turns per dialogue.
\end{itemize}

\paragraph{Prompts} As shown in Figure~\ref{fig:prompt} and Figure~\ref{fig:fewshot-prompt}, we present the prompt templates used in our experiments. Retrieved documents and contextual examples are inserted into \texttt{Document [1]} and \texttt{Example [1]} positions after reordering by different strategies to compare their impact on performance.

\begin{figure}[htbp]
    \centering
    \includegraphics[width=1\linewidth]{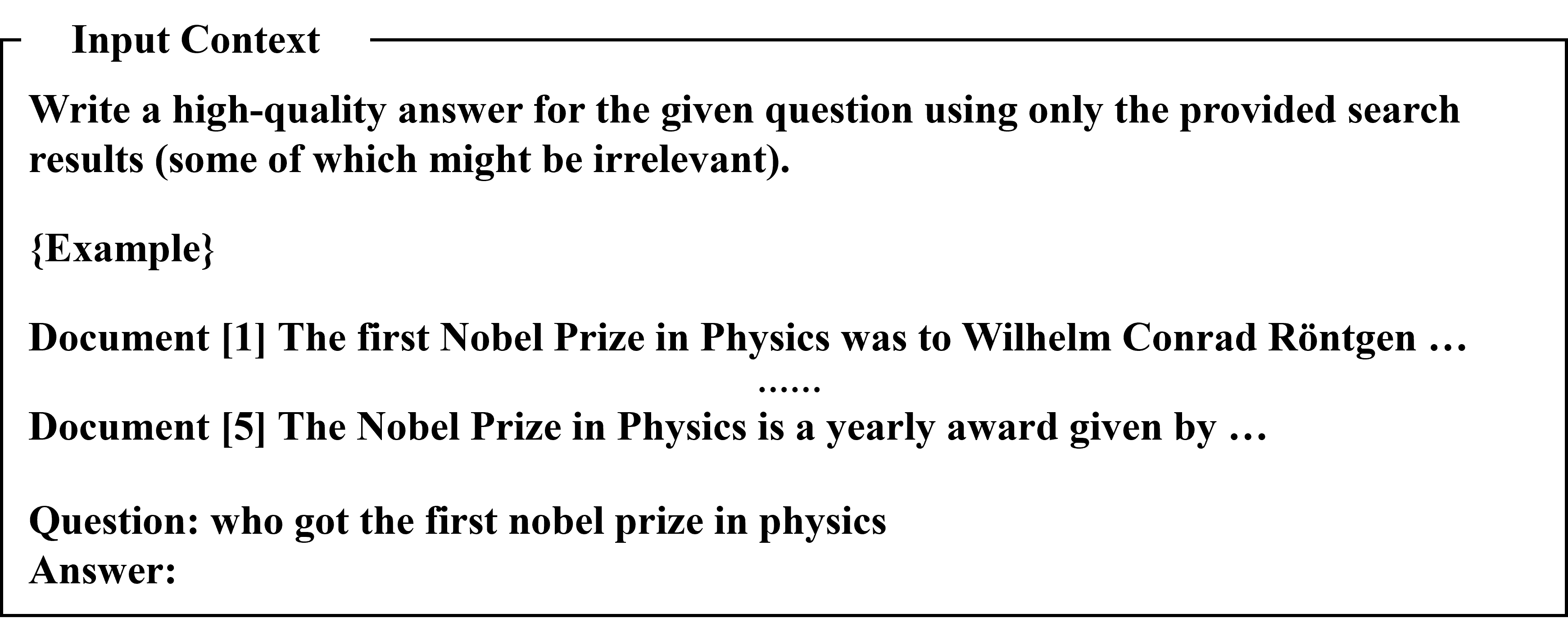}
    \caption{Prompt for multi-hop QA experiment.}
    \label{fig:prompt}
\end{figure}

\begin{figure}[htbp]
    \centering
    \includegraphics[width=1\linewidth]{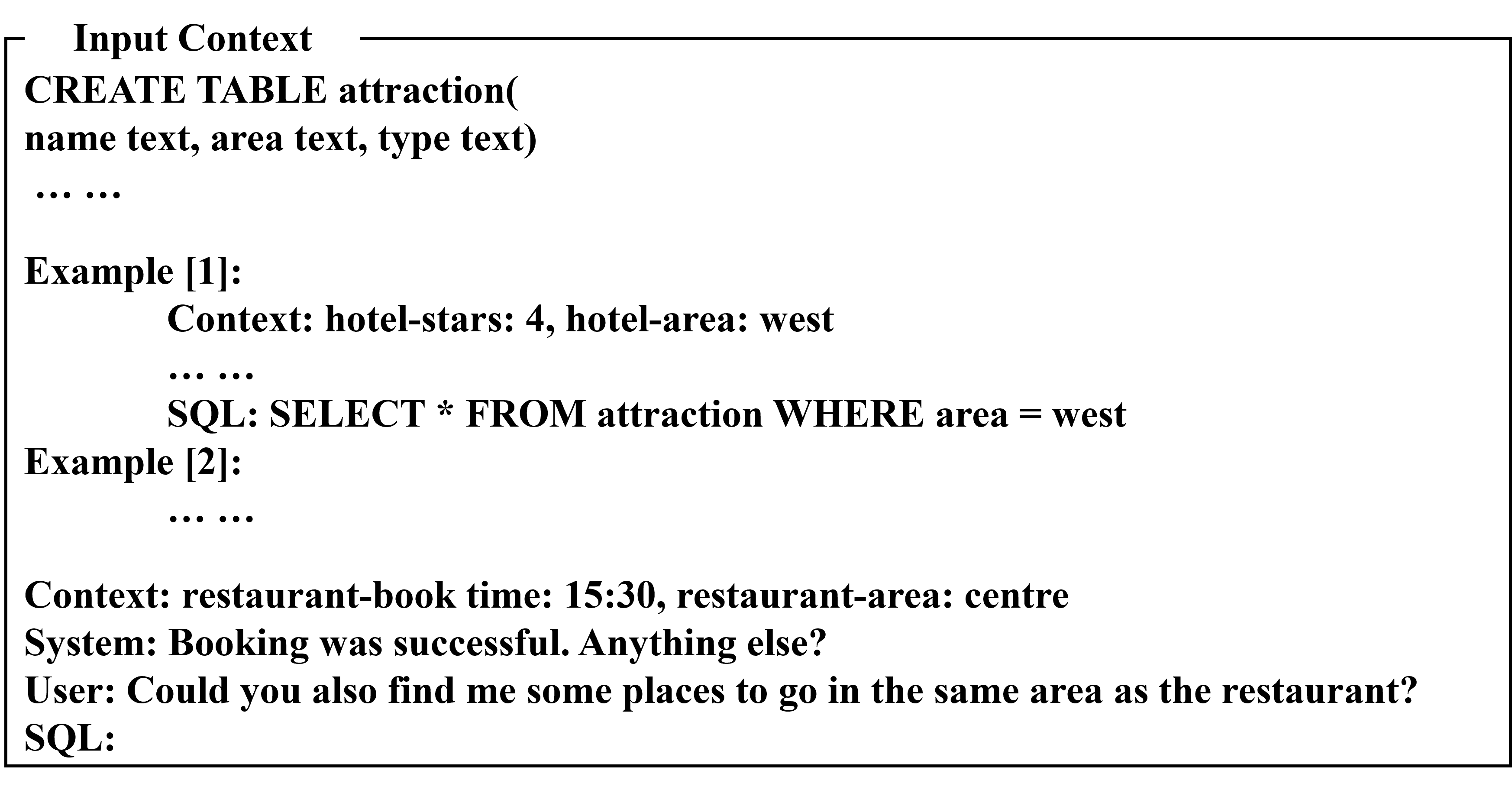}
    \caption{Prompt for dew-shot experiment.}
    \label{fig:fewshot-prompt}
\end{figure}

As shown in Figure~\ref{fig:nostruct_prompt}, we provide the prompt template used in the analysis with disrupted input structure in Section~\ref{sec:observation} . Twenty documents are randomly shuffled and inserted, with all punctuation removed and uppercase letters converted to lowercase to destroy structural cues.

\begin{figure}[h]
    \centering
    \includegraphics[width=1\linewidth]{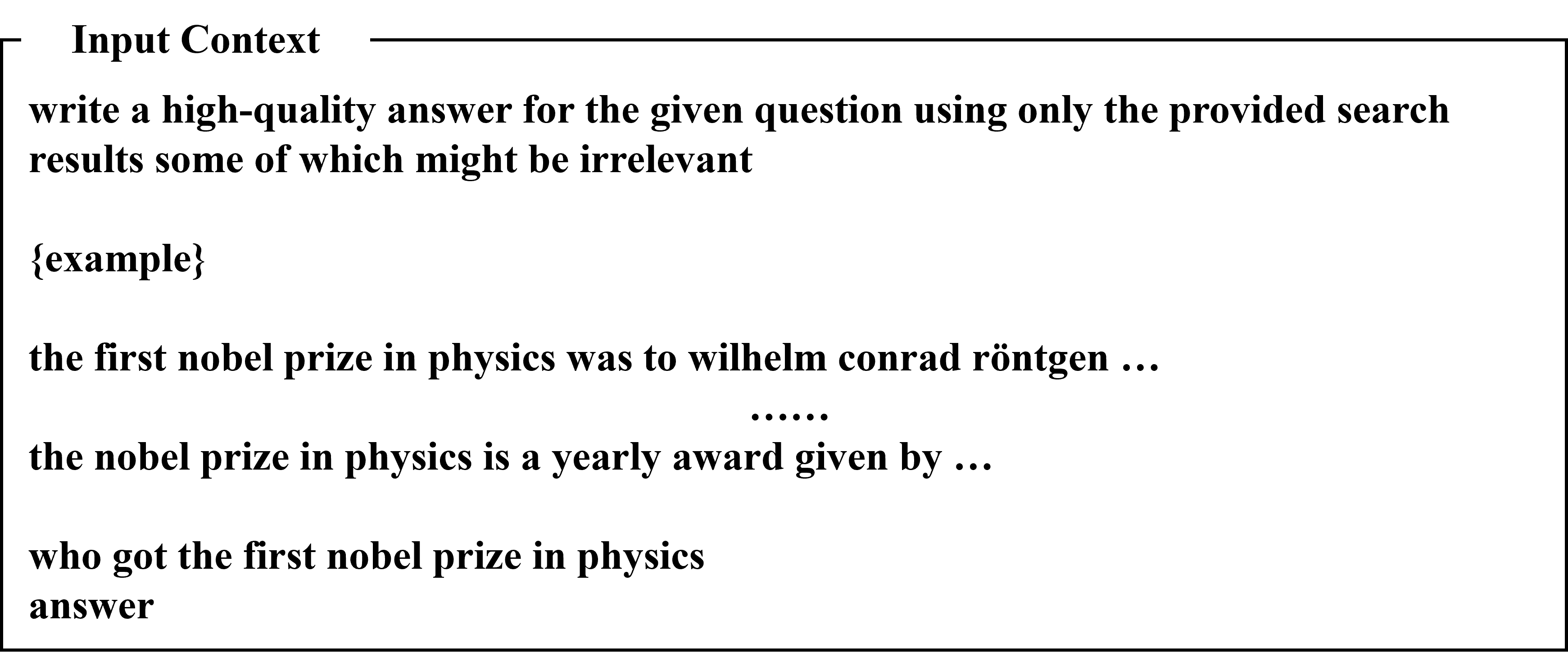}
    \caption{Prompt after disrupting the structural blocks}
    \label{fig:nostruct_prompt}
\end{figure}

\section{More details of how LLMs distribute their attention.}
\label{sec:appendix-more-attn}
Figures~\ref{fig:two_figures1},~\ref{fig:two_figures2}, and~\ref{fig:two_figures3} show the attention distributions across different models and document counts, illustrating that the attention basin phenomenon is consistently observed.

\begin{figure}[ht]
    \centering
    \begin{subfigure}[b]{0.45\textwidth}
        \centering
        \includegraphics[width=\textwidth]{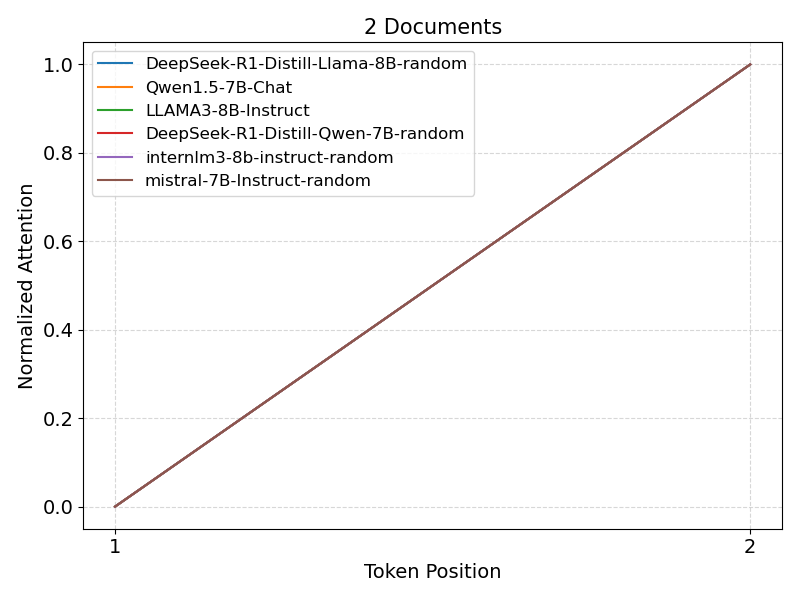}
        \caption{Two documents}
    \end{subfigure}
    \hfill
    \begin{subfigure}[b]{0.45\textwidth}
        \centering
        \includegraphics[width=\textwidth]{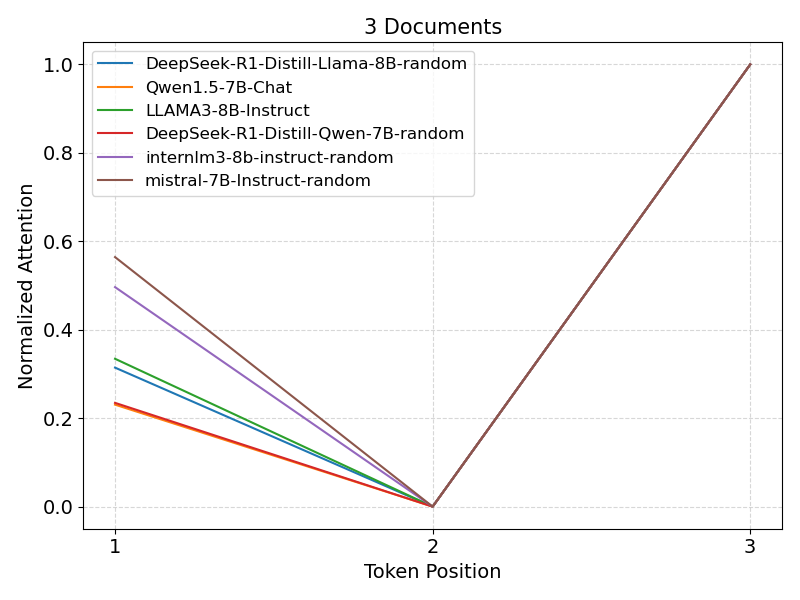}
        \caption{Three documents}
    \end{subfigure}
    \caption{Attention distribution with 2 and 3 documents}
    \label{fig:two_figures1}
\end{figure}

\begin{figure}[!h]
    \centering
    \begin{subfigure}[b]{0.45\textwidth}
        \centering
        \includegraphics[width=\textwidth]{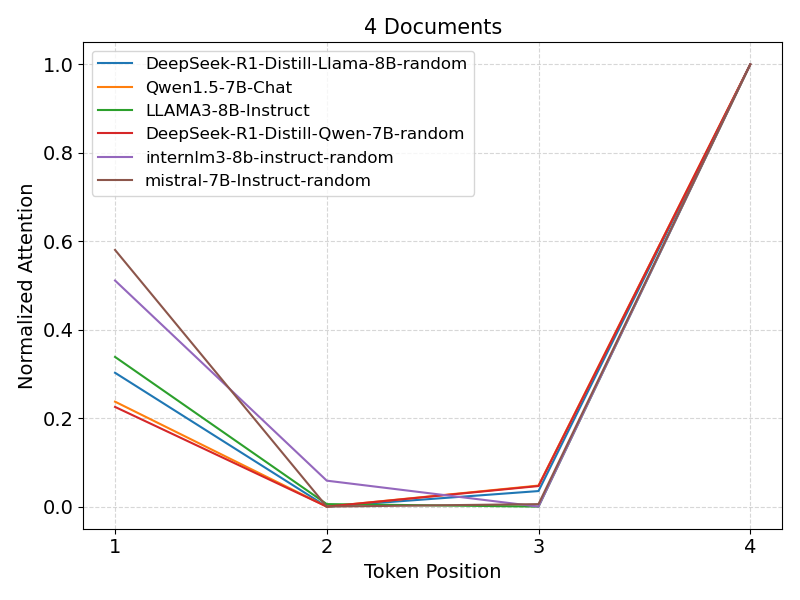}
        \caption{Four documents}
    \end{subfigure}
    \hfill
    \begin{subfigure}[b]{0.45\textwidth}
        \centering
        \includegraphics[width=\textwidth]{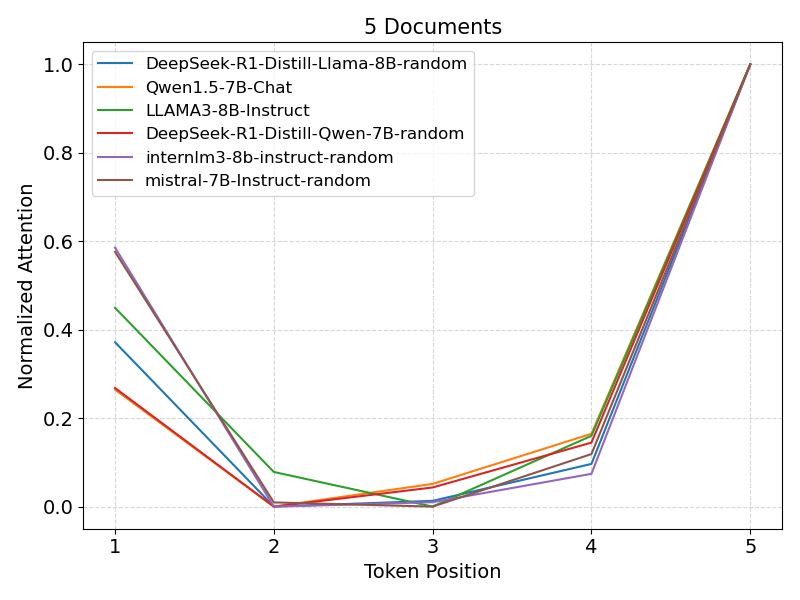}
        \caption{Five documents}
    \end{subfigure}
    \caption{Attention distribution with 4 and 5 documents}
    \label{fig:two_figures2}
\end{figure}

\begin{figure}[!h]
    \centering
    \begin{subfigure}[b]{0.45\textwidth}
        \centering
        \includegraphics[width=\textwidth]{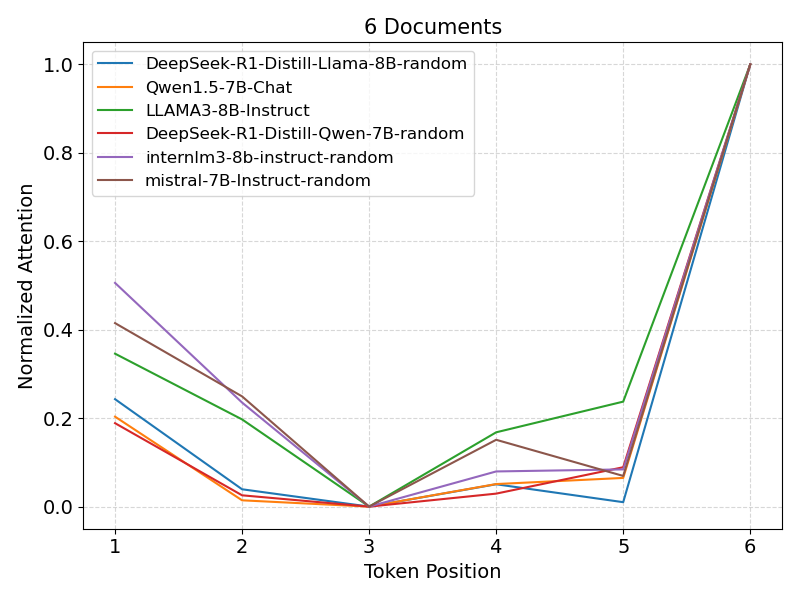}
        \caption{Six documents}
    \end{subfigure}
    \hfill
    \begin{subfigure}[b]{0.45\textwidth}
        \centering
        \includegraphics[width=\textwidth]{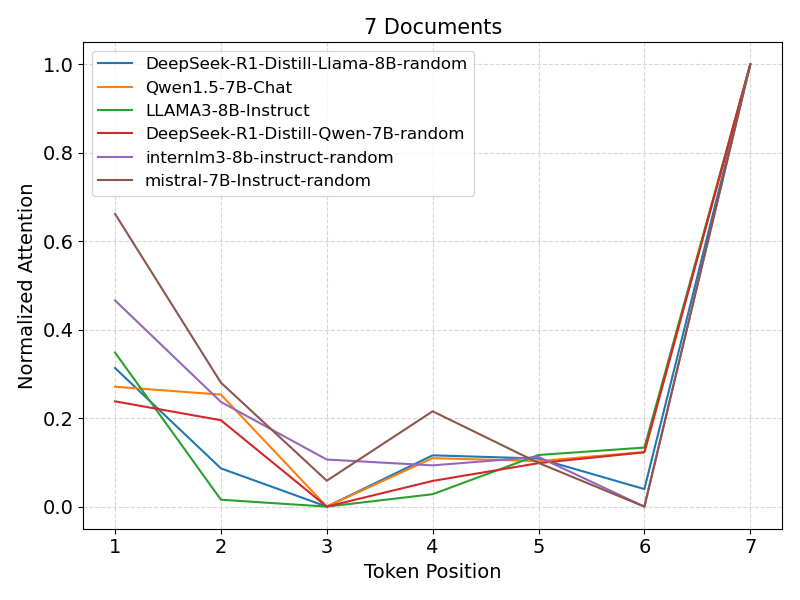}
        \caption{Seven documents}
    \end{subfigure}
    \caption{Attention distribution with 6 and 7 documents}
    \label{fig:two_figures3}
\end{figure}

\section{Theoretical analysis of attention-guided reranking in long-context tasks}
\label{main proof}

\subsection{Optimization objective}
\label{e1}
Given a long-context input $S = \{t, d_1, \ldots, d_n, q\}$, where $t$ is a fixed prompt template, $q$ is the user query, and $D = \{d_1, \ldots, d_n\}$ denotes information segments containing large numbers of tokens, such as knowledge documents in multi-hop QA or in-context examples in few-shot tasks. Let $ d^* $ denote the correct document, we aim to demonstrate that strategically positioning $ d^* $ in high-attention regions of the input sequence increases its cross-layer average attention weight $ \bar{\alpha}_{d^*} $, thereby maximizing the posterior probability $ P(y^*|x,D) $ of generating correct answer $ y^* $. Formally, we seek to prove:
\begin{equation}
    \bar{\alpha}_{d^*} > \bar{\alpha}_{d_j}\ (\forall d_j \neq d^*) \implies P(y^*|x,D) > P(y|d_j,x,D)
\end{equation}

\subsection{Key assumptions}
\label{e2}

\begin{assumption}[Orthogonal Semantic Representation] \label{ass:ortho}
Let $ \{e_{d_k}\} $ denote document-level embeddings. These satisfy pairwise orthogonality:
\begin{equation}
    \langle e_{d_i}, e_{d_j} \rangle = 0\quad \forall i \neq j
\end{equation}
with $ \|e_{d_k}\|_2 = 1 $ for normalization.
\end{assumption}

\begin{assumption}[Position-Attention Coupling] \label{ass:pos-attn}
Let $ \alpha^{(l)}_p $ denote attention weight for position $ p $ at layer $ l $. There exists position-dependent bias:
\begin{equation}
    \mathbb{E}[\alpha^{(l)}_{p}] = f(p) + \epsilon^{(l)}_p
\end{equation}
where $ f: \mathbb{N} \to \mathbb{R}^+ $ is a U-shaped function modeling the attention basin phenomenon, and $ \epsilon^{(l)}_p $ represents position-agnostic content-based attention. 
\end{assumption}

\begin{assumption}[Compositional Weight Binding] \label{ass:composition}
The output projection matrix $ W_y $ decomposes as:
\begin{equation}
    W_y = [E_d \| E_t] \cdot W_c
\end{equation}
where $ E_d $ is the document embedding matrix, $ E_t $ the token embedding matrix, and $ W_c $ a composition matrix satisfying $ \|W_c\|_F \leq \gamma $.
\end{assumption}

\subsection{Technical lemmas}
\label{e3}
\begin{lemma}[Hidden State Composition] \label{lem:composition}
The final hidden state $ h_{\text{last}} $ decomposes into document-aware components:
\begin{equation}
    h_{\text{last}} = h_{\text{init}} + \sum_{k=1}^n \left(\sum_{l=1}^L \alpha^{(l)}_{d_k} v^{(l)}_{d_k}\right) + \Delta h_{\text{noise}}
\end{equation}
where $ \alpha^{(l)}_{d_k} = \sum_{p \in \mathcal{P}(d_k)} \alpha^{(l)}_p $ aggregates position-wise attention for document $ d_k $, and $ v^{(l)}_{d_k} $ represents its value vector at layer $ l $.
\end{lemma}

\begin{proof}
Through residual connections, each layer's output accumulates document-specific contributions:
\begin{equation}
    h^{(l)} = h^{(l-1)} + \text{Attn}^{(l)}(h^{(l-1)}) = h^{(0)} + \sum_{m=1}^l \text{Attn}^{(m)}(h^{(m-1)})
\end{equation}
Decompose attention heads into document-level components:
\begin{equation}
    \text{Attn}^{(m)} = \sum_{k=1}^n \alpha^{(m)}_{d_k} v^{(m)}_{d_k} + \text{CrossDoc}^{(m)}
\end{equation}

Under Assumption \ref{ass:ortho}, cross-document terms $ \text{CrossDoc}^{(m)} $ become negligible due to orthogonality, yielding the stated decomposition.
\end{proof}

\begin{lemma}[Logit Formation] \label{lem:logit}
The logit for answer token $ y^* $ decomposes as:
\begin{equation}
    \text{logit}(y^*) = \underbrace{\langle h_{\text{last}}, e_{d^*} \rangle}_{\text{document term}} + \underbrace{\langle h_{\text{last}}, e_{y^*} \rangle}_{\text{token term}} + b_{y^*}
\end{equation}
where $ e_{d^*} $ and $ e_{y^*} $ are orthogonal components from Assumption \ref{ass:composition}.
\end{lemma}

\begin{proof}
Using Assumption \ref{ass:composition}, the output projection becomes:
\begin{equation}
    W_y h_{\text{last}} = [E_d \| E_t] W_c h_{\text{last}} = E_d (W_c^{(d)} h_{\text{last}}) + E_t (W_c^{(t)} h_{\text{last}})
\end{equation}

Thus for target token $ y^* $ embedded as $ e_{y^*} = E_t[i] $, the logit contains separate document and token alignment terms.
\end{proof}

\subsection{Main proposition}
\label{e4}
\begin{proposition}[Attention-Probability Monotonicity]
\label{prop:mono}
For documents $ d^* $ and $ d_j $ with average attention weights $ \bar{\alpha}_{d^*} = \frac{1}{L}\sum_{l=1}^L \alpha^{(l)}_{d^*} $ and $ \bar{\alpha}_{d_j} $, if $ \bar{\alpha}_{d^*} > \bar{\alpha}_{d_j} $, then:
\begin{equation}
    \frac{\partial P(y^*|x,D)}{\partial \bar{\alpha}_{d^*}} > \frac{\partial P(y^*|x,D)}{\partial \bar{\alpha}_{d_j}} \geq 0
\end{equation}
\end{proposition}

\begin{proof} 

\textbf{Step 1: Document-term dominance}\\
From Lemma \ref{lem:composition} and \ref{lem:logit}, the document alignment term dominates when $ d^* $ contains sufficient answer evidence:
\begin{equation}
    \langle h_{\text{last}}, e_{d^*} \rangle = \bar{\alpha}_{d^*} \underbrace{\left\langle \frac{1}{L}\sum_{l=1}^L v^{(l)}_{d^*}, e_{d^*} \right\rangle}_{\kappa} + \mathcal{O}(\max_{k \neq *} \bar{\alpha}_{d_k})
\end{equation}
where $ \kappa > 0 $ due to value-key alignment in transformers.

\textbf{Step 2: Probability gradient analysis}\\
The output probability computes as:
\begin{equation}
    P(y^*|x,D) = \frac{\exp(\text{logit}(y^*))}{\sum_y \exp(\text{logit}(y))}
\end{equation}

Taking partial derivative with respect to $ \bar{\alpha}_{d^*} $:
\begin{equation}
    \frac{\partial P}{\partial \bar{\alpha}_{d^*}} = P(y^*|x,D)(1 - P(y^*|x,D)) \kappa > 0
\end{equation}

Similarly, $ \frac{\partial P}{\partial \bar{\alpha}_{d_j}} = -P(1-P)\kappa_j \leq 0 $ for $ j \neq * $.

\textbf{Step 3: Strict monotonicity}\\
Given $ \kappa \propto \|v_{d^*}\| \cos\theta_{v_{d^*},e_{d^*}} $ and Assumption \ref{ass:ortho}, $ \cos\theta = 1 $. Thus:
\begin{equation}
    \bar{\alpha}_{d^*} > \bar{\alpha}_{d_j} \implies \frac{\partial P}{\partial \bar{\alpha}_{d^*}} > \left|\frac{\partial P}{\partial \bar{\alpha}_{d_j}}\right|
\end{equation}

 Hence complete the proof.
\end{proof}

\subsection{Implications for long-context tasks}
\label{e5}
\begin{corollary}[Optimal Document Positioning]
\label{cor:doc-pos}
Let $ \mathcal{P}_{\text{opt}} $ denote positions with maximal attention basin effect in Assumption \ref{ass:pos-attn}. Placing $ d^* $ at $ p^* \in \mathcal{P}_{\text{opt}} $ maximizes $ \bar{\alpha}_{d^*} $, leading to:
\begin{equation}
    \mathbb{E}[P(y^*|x,D)]_{p^*} \geq \mathbb{E}[P(y^*|x,D)]_{p} \quad \forall p \notin \mathcal{P}_{\text{opt}}
\end{equation}
\end{corollary}

\begin{proof}
From Assumption \ref{ass:pos-attn}, positions in $ \mathcal{P}_{\text{opt}} $ maximize $ \mathbb{E}[\alpha^{(l)}_p] $. By the monotonicity in Proposition 1, positioning $ d^* $ at $ p^* $ achieves:
\begin{equation}
    \mathbb{E}[\bar{\alpha}_{d^*}|p^*] = \frac{1}{L}\sum_{l=1}^L f(p^*) + \mathbb{E}[\epsilon^{(l)}_{p^*}] > \frac{1}{L}\sum_{l=1}^L f(p) + \mathbb{E}[\epsilon^{(l)}_p]
\end{equation}
for any $ p \notin \mathcal{P}_{\text{opt}} $. Proposition 1 then guarantees higher $ P(y^*|x,D) $.
\end{proof}

\begin{corollary}[Layer-wise Attention Degradation] \label{cor:layer}
Let $ \rho(l) = \frac{\mathbb{V}[\epsilon^{(l)}_p]}{\mathbb{V}[f(p)]} $ measure the relative strength of content-based vs position-based attention at layer $ l $. There exists a layer depth threshold $ L^* $ such that:
\begin{align*}
\rho(l) &< 1\quad \forall l < L^* \quad \textit{(position-dominated regime)} \\
\rho(l) &\geq 1\quad \forall l \geq L^* \quad \textit{(content-dominated regime)}
\end{align*}
Thus, shallow layers better preserve positional bias patterns from Assumption \ref{ass:pos-attn}, while deeper layers exhibit attenuated positional effects.
\end{corollary}

\begin{proof}
From Assumption \ref{ass:pos-attn}, decompose attention variance:
\begin{equation}
    \mathbb{V}[\alpha^{(l)}_p] = \underbrace{\mathbb{V}[f(p)]}_{\text{positional}} + \underbrace{\mathbb{V}[\epsilon^{(l)}_p]}_{\text{content-based}}
\end{equation}

Transformer architectures typically exhibit increasing $ \mathbb{V}[\epsilon^{(l)}_p] $ with depth as self-attention becomes more semantic. Solving $ \mathbb{V}[f(p)] = \mathbb{V}[\epsilon^{(l)}_p] $ yields threshold $ L^* $ where positional effects become subdominant.
\end{proof}
\subsection{Remark: connecting theory to method}
\label{e6}
Our theoretical analysis rigorously justifies the core principle of attention-guided document reranking:
\begin{itemize}[leftmargin=*]
    \item The U-shaped attention basin (Assumption \ref{ass:pos-attn}) explains the \textit{lost-in-the-middle} phenomenon through its positional expectation $ \mathbb{E}[\alpha^{(l)}_p] $
    \item The attention-probability monotonicity (Proposition 1) formally establishes that boosting $ \bar{\alpha}_{d^*} $ via strategic positioning directly increases answer correctness
    \item Corollary \ref{cor:doc-pos} provides theoretical guarantees for our method's effectiveness: reranking documents to place $ d^* $ in $ \mathcal{P}_{\text{opt}} $ (typically sequence edges) maximizes its attention influence
    \item Corollary \ref{cor:layer} reveals the layer-dependent nature of attention guidance: shallow layers' position-dominated regime ($\rho(l)<1$) better preserves document ordering signals critical for reranking, while deeper layers' content-focused attention ($\rho(l)\geq1$) introduces positional ambiguity. This theoretically justifies why using shallower attention maps produces better reranking results.
\end{itemize}

This mathematical foundation not only explains empirical observations but also guides future extensions: the composition matrix $ W_c $ in Assumption \ref{ass:composition} suggests directions for training-based attention shaping, while the orthogonality in Assumption \ref{ass:ortho} motivates improved document encoding schemes to better satisfy theoretical prerequisites.




\begin{figure}[!h]
    \centering
    \begin{subfigure}[b]{0.4\textwidth}
        \centering
        \includegraphics[width=\textwidth]{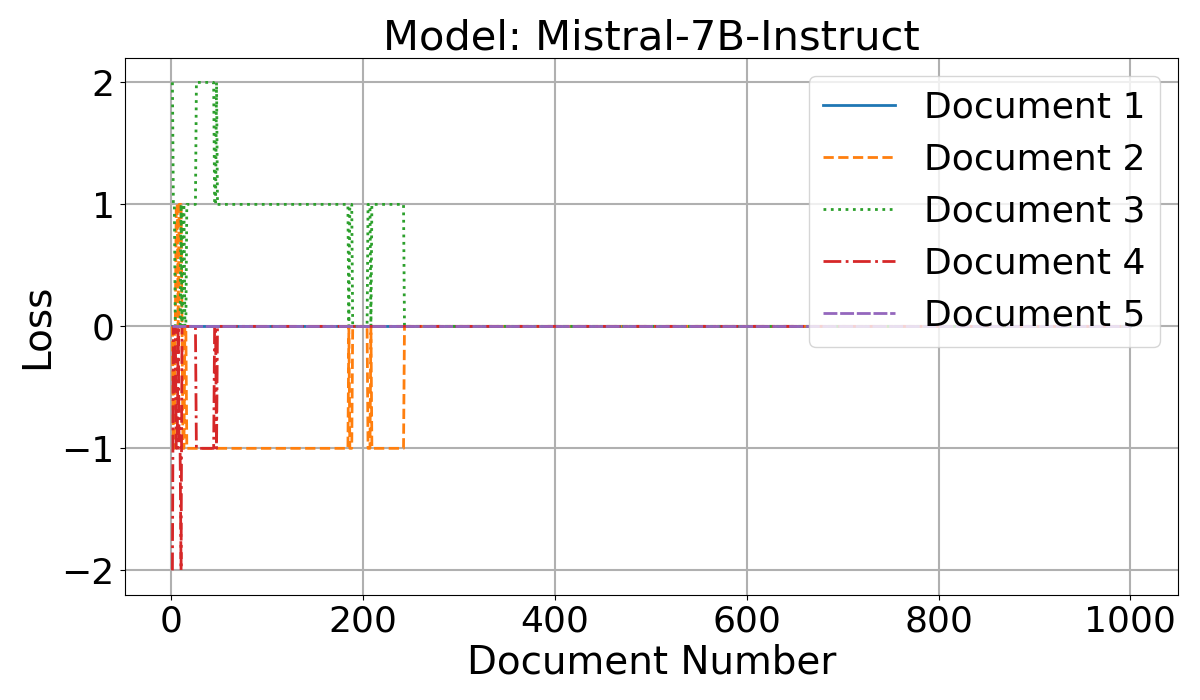}
        \caption{Mistral-7B-Instruct}
    \end{subfigure}
    \hfill
    \begin{subfigure}[b]{0.4\textwidth}
        \centering
        \includegraphics[width=\textwidth]{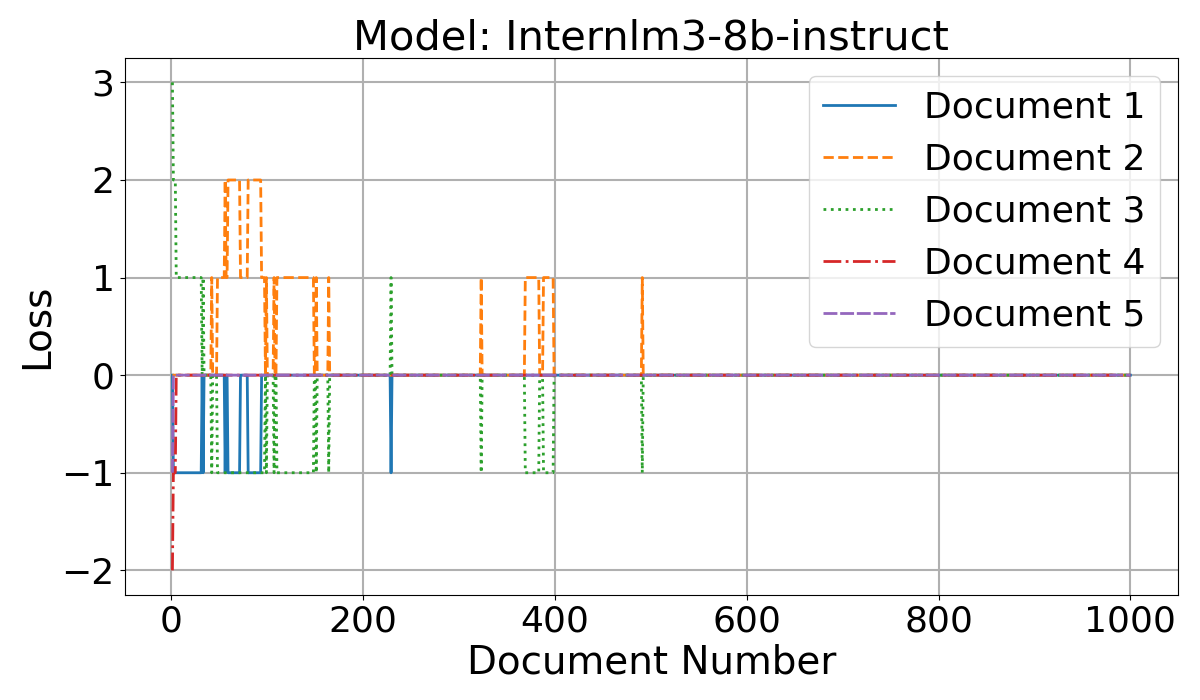}
        \caption{Internlm3-8b-instruct}
    \end{subfigure}
    \caption{The relationship between long-context data volume and attention distribution in LLMs.}
    \label{nums_1}
\end{figure}

\section{How many documents are required?}
\label{how many}
We design an experiment to determine how much data is required for attention patterns to converge. Following the setup in Section~\ref{exp1}, we incrementally increase the number of samples and compare the resulting attention distributions to those from the full dataset. As shown in Figures~\ref{nums_1} and~\ref{nums_2}, all models converge after about 400 samples. Notably, we focus on convergence at the document boundaries—the first (blue) and last (purple) documents. In DeepSeek-LLM and InternLM3-8B-Instruct, boundary attention converges with only ~200 samples. Mistral-7B-Instruct and Qwen2.5-7B match the final pattern from the start. These results suggest that, for some models, a single sample may suffice to approximate full-context attention behavior.

\begin{figure}[h]
    \centering
    \begin{subfigure}[b]{0.4\textwidth}
        \centering
        \includegraphics[width=\textwidth]{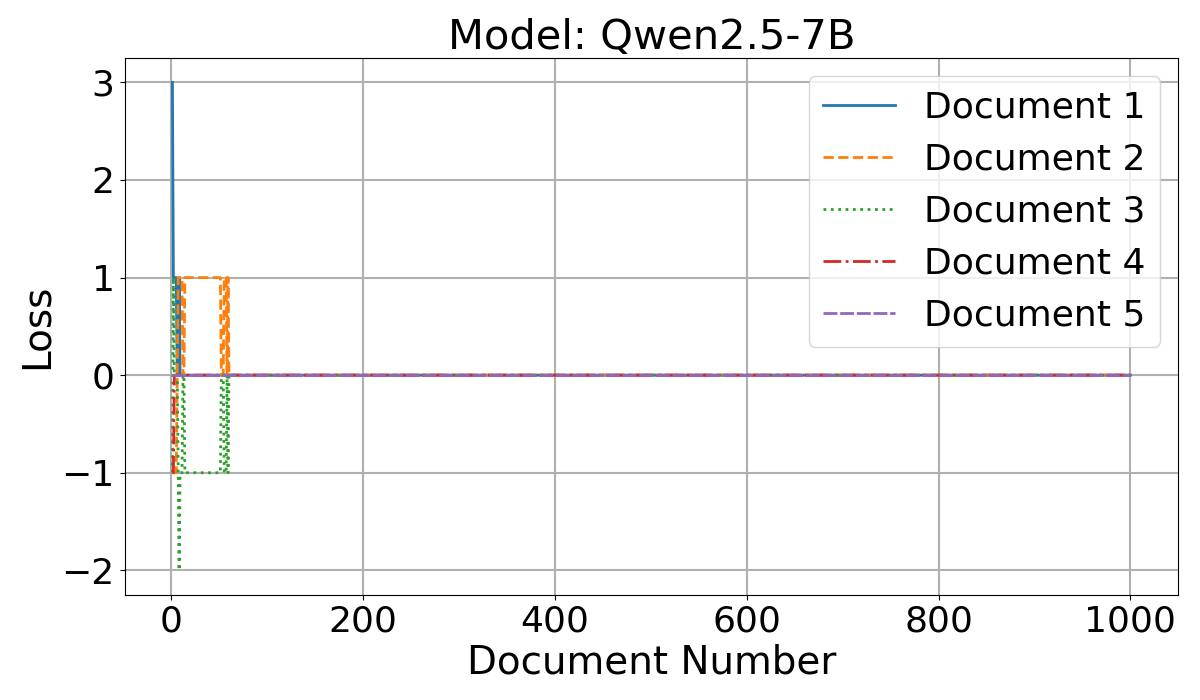}
        \caption{Qwen2.5-7B}
    \end{subfigure}
    \hfill
    \begin{subfigure}[b]{0.4\textwidth}
        \centering
        \includegraphics[width=\textwidth]{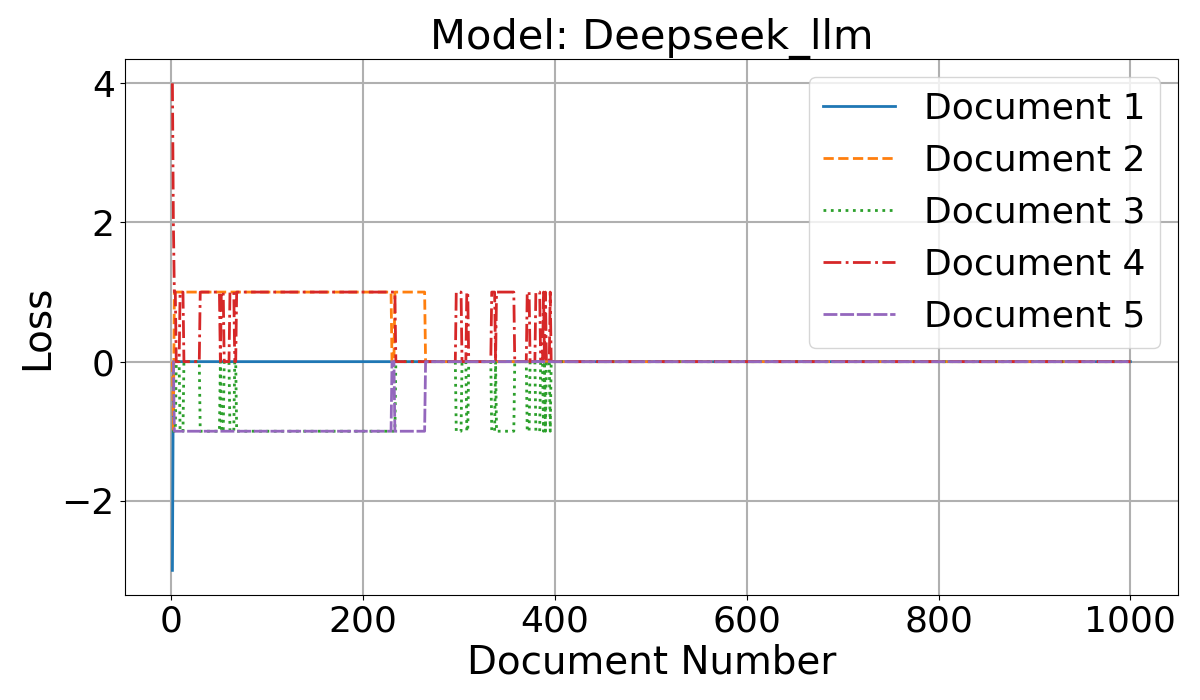}
        \caption{Deepseek-llm}
    \end{subfigure}
    \caption{The relationship between sample size and attention distribution in LLMs.}
    \label{nums_2}
\end{figure}

\section{Attention distribution experiments and case studies}
\label{add exp}

Following the setup in Experiment~\ref{exp1}, we analyze the average attention scores assigned to ground-truth and noise documents under different reranking baselines. As shown in Table~\ref{tab:attention-scores}, \ours assigns the highest average attention score to relevant documents and the lowest to noise ones, demonstrating that \ours effectively guides the model to focus on the most critical documents.

\begin{table*}[!h]
\centering
\caption{Average attention scores on ground-truth and noise documents under different reranking strategies.}
\label{tab:attention-scores}
\resizebox{\textwidth}{!}{%
\begin{tabular}{lccccc|ccccc}
\toprule
 & \multicolumn{5}{c|}{Correct Documents $\uparrow$} & \multicolumn{5}{c}{Wrong Documents $\downarrow$} \\
Model & Random & Descending & Ascending & LIM & Rerank & Random & Descending & Ascending & LIM & \ours \textbf{(ours)} \\
\cmidrule(lr){2-6} \cmidrule(lr){7-11}
DeepSeek-R1-Distill-Llama-8B
  & 0.0182 & 0.0183 & \textbf{0.0194} & 0.0183 & \underline{0.0193}
  & 0.0048 & 0.0048 & \underline{0.0043} & 0.0048 & \textbf{0.0042} \\

LLAMA3-8B-Instruct
  & 0.0212 & 0.0212 & \underline{0.0230} & 0.0214 & \textbf{0.0232}
  & 0.0054 & 0.0054 & \underline{0.0045} & 0.0053 & \textbf{0.0043} \\

LLAMA3-8B
  & 0.0238 & 0.0238 & \textbf{0.0258} & 0.0238 & \underline{0.0256}
  & 0.0062 & 0.0062 & \underline{0.0053} & 0.0062 & \textbf{0.0052} \\

mistral-7B-Instruct
  & 0.0230 & 0.0229 & \underline{0.0238} & 0.0231 & \textbf{0.0244}
  & 0.0059 & 0.0060 & \underline{0.0056} & 0.0059 & \textbf{0.0051} \\

deepseek-llm
  & 0.0389 & 0.0387 & \underline{0.0401} & 0.0386 & \textbf{0.0404}
  & 0.0110 & 0.0111 & \underline{0.0106} & 0.0112 & \textbf{0.0103} \\

internlm3-8b-instruct
  & 0.0177 & 0.0178 & \underline{0.0181} & 0.0177 & \textbf{0.0185}
  & 0.0045 & 0.0046 & \underline{0.0044} & 0.0045 & \textbf{0.0041} \\

qwen 2.5 1.5B
  & 0.0118 & 0.0119 & \textbf{0.0128} & 0.0118 & \underline{0.0127}
  & 0.0030 & 0.0031 & \textbf{0.0026} & 0.0030 & \textbf{0.0026} \\

qwen 2.5 3B
  & 0.0169 & 0.0168 & \underline{0.0170} & 0.0166 & \textbf{0.0183}
  & \underline{0.0043} & \underline{0.0043} & 0.0044 & 0.0044 & \textbf{0.0034} \\

qwen 2.5 7B
  & 0.0283 & 0.0285 & \textbf{0.0305} & 0.0284 & \underline{0.0304}
  & 0.0072 & 0.0072 & \textbf{0.0061} & 0.0072 & \textbf{0.0061} \\

\midrule
\textbf{Average}
  & 0.022200 & 0.022211 & \underline{0.023389} & 0.022189 & \textbf{0.023644}
  & 0.005811 & 0.005856 & \underline{0.005311} & 0.005833 & \textbf{0.005033} \\
\bottomrule
\end{tabular}%
}
\end{table*}

\begin{figure*}[!h]
    \centering
    \includegraphics[width=0.7\linewidth]{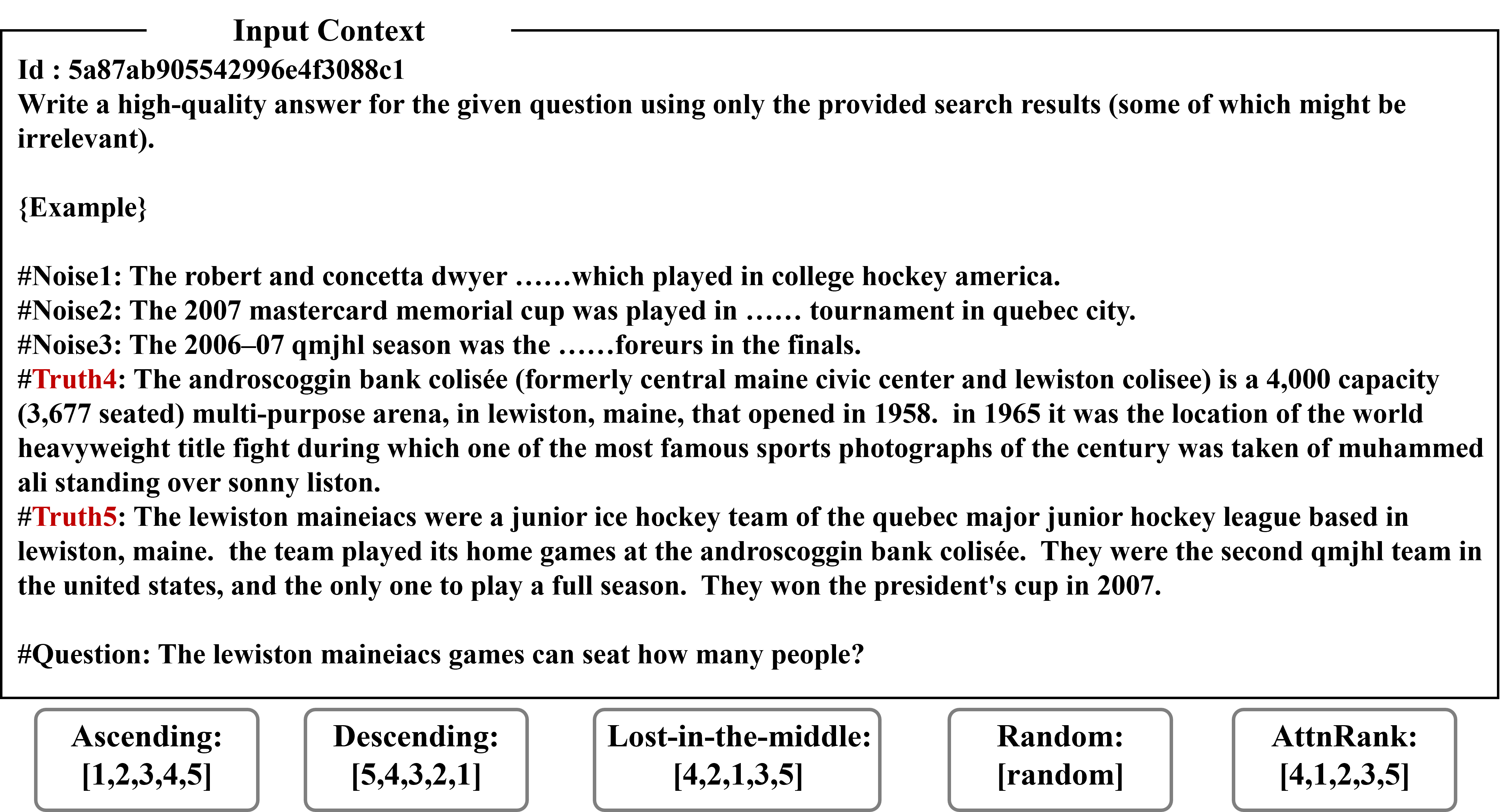}
    \caption{Case 1's input prompt and ranking outcomes under different reranking strategies.
}
    \label{fig:case2}
\end{figure*}

We further conducted three case studies.  Figures  \ref{fig:case1}–\ref{fig:case2}, \ref{fig:case21}–\ref{fig:case22} and \ref{fig:case31}–\ref{fig:case32} present two additional case studies. \ours robustly maintains high attention on critical documents and low attention on noise, validating its effectiveness. In contrast, the descending strategy preserves high attention for relevant documents but also highlights noise, and the lost-in-the-middle approach reduces noise attention at the expense of under-attending to relevant documents. \ours consistently focuses on important content while disregarding noise.

\begin{figure*}
    \centering
    \includegraphics[width=0.8\linewidth]{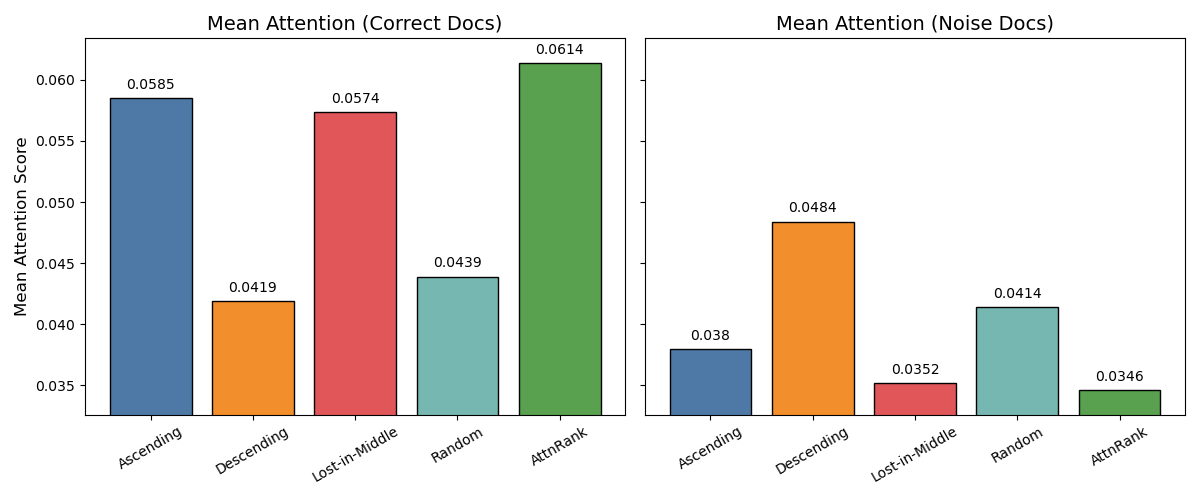}
    \caption{Average attention scores for relevant and noise documents under various reranking strategies in case 1. \ours attains the highest attention on relevant documents and the lowest on irrelevant documents, validating its effectiveness.
}
    \label{fig:case1}
\end{figure*}
\begin{figure*}
    \centering
    \includegraphics[width=0.7\linewidth]{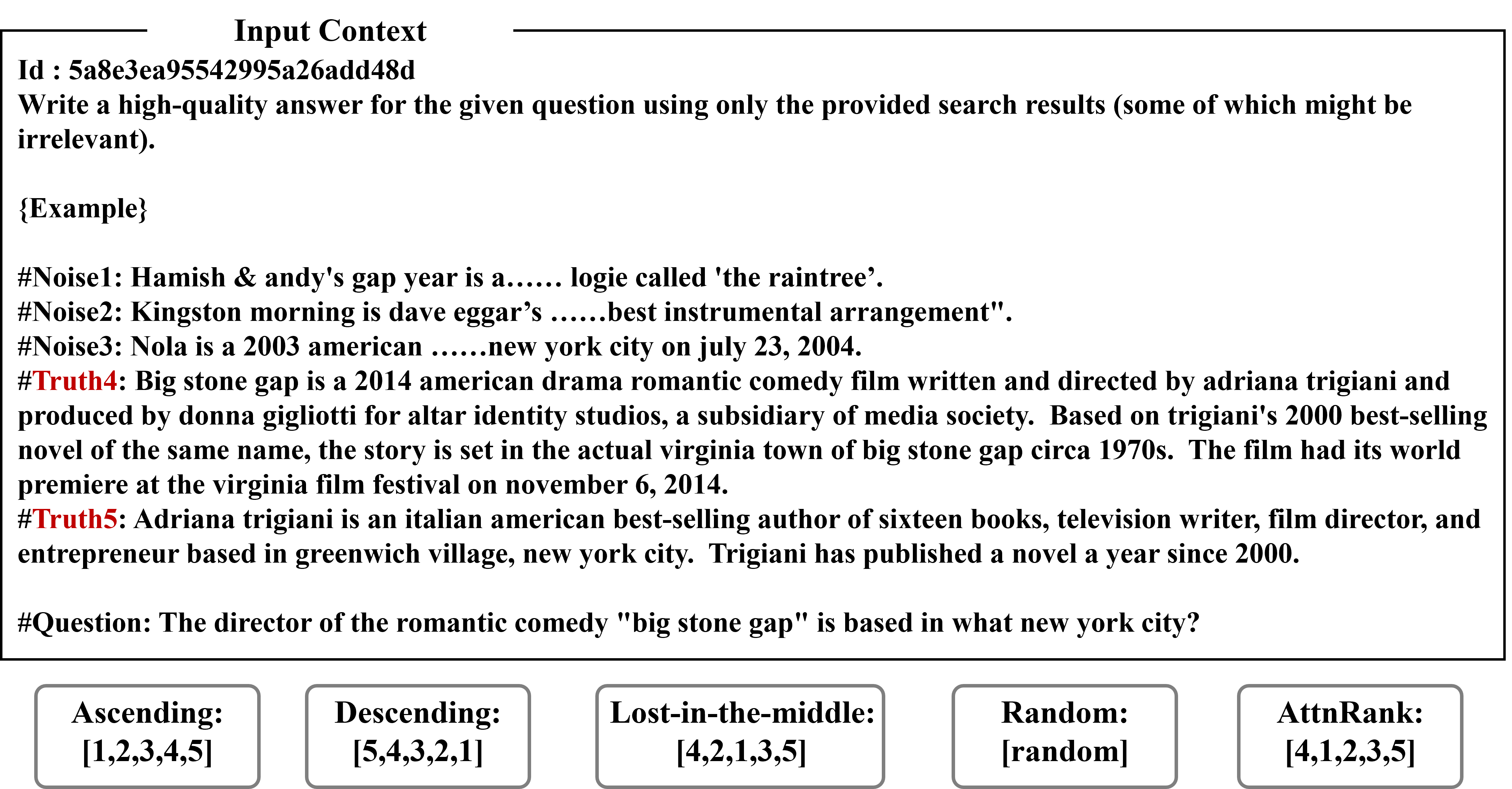}
    \caption{Case 2's input prompt and ranking outcomes under different reranking strategies.}
    \label{fig:case21}
\end{figure*}

\begin{figure*}
    \centering
    \includegraphics[width=0.8\linewidth]{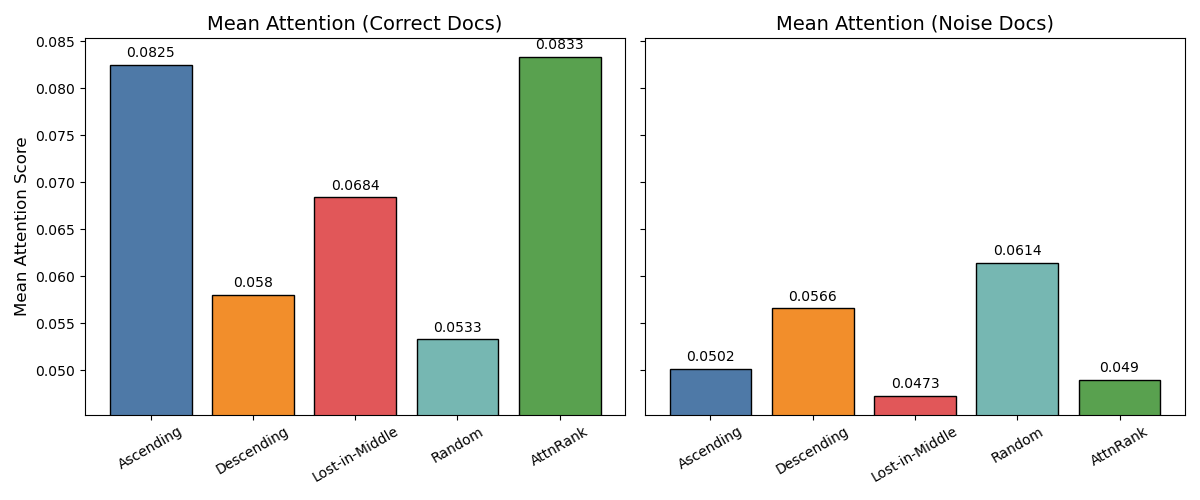}
    \caption{Average attention scores for relevant and noise documents in case 2.}
    \label{fig:case22}
\end{figure*}

\begin{figure*}
    \centering
    \includegraphics[width=0.8\linewidth]{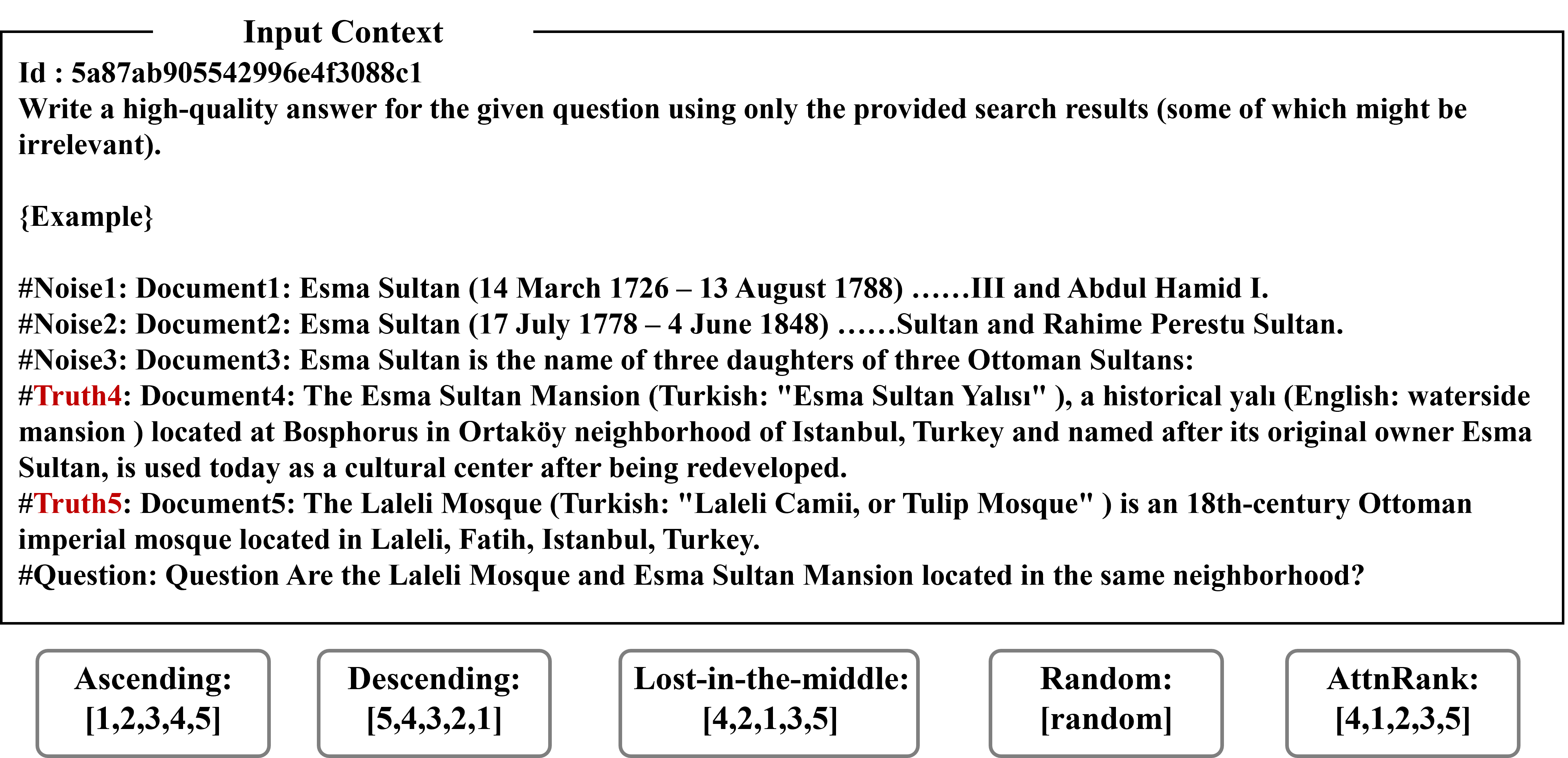}
    \caption{Case 3's input prompt and ranking outcomes under different reranking strategies.}
    \label{fig:case31}
\end{figure*}

\begin{figure*}
    \centering
    \includegraphics[width=0.8\linewidth]{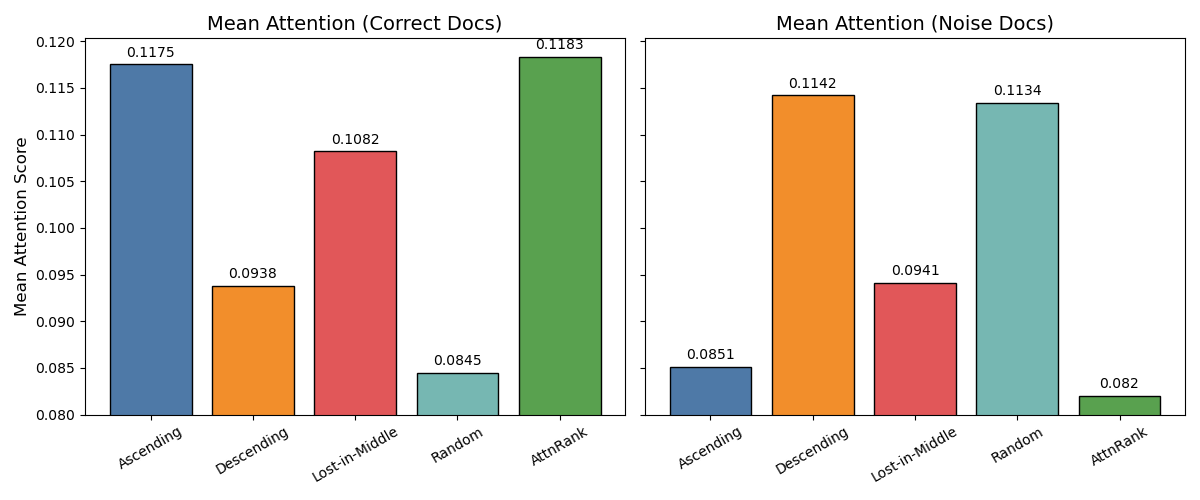}
    \caption{Average attention scores for relevant and noise documents in case 3.}
    \label{fig:case32}
\end{figure*}

\section{Limitation and future works}
\label{limitation}

The proposed \ours method effectively mitigates detrimental positional biases by aligning document relevance with the model's intrinsic attention patterns, thereby unlocking significant performance gains. However, due to the fact that most existing closed-source models do not expose attention scores through their APIs, the effectiveness of our approach on such models cannot be verified at present. This limitation highlights the need for future work on developing re-ranking strategies compatible with closed-source LLMs. 
Moreover, in many applications, it is desirable for the model to attend equally to content across different positions. We therefore suggest that future research explore methods to mitigate the \emph{attention basin} phenomenon, enabling more uniform attention distribution across different tokens.

\end{document}